\documentclass{article}
\usepackage[utf8]{inputenc}

\usepackage[margin=1in]{geometry}

\usepackage{amsthm}
\usepackage{amsmath}
\usepackage{amssymb,epsfig}
\usepackage{graphicx}
\usepackage{epstopdf}
\usepackage{comment}
\usepackage{array}
\usepackage{algorithm}
\usepackage{url}
\usepackage{enumitem}   
\usepackage{lscape}
\usepackage{algpseudocode}
\usepackage{setspace}
\usepackage{multicol}
\usepackage{multirow}
\usepackage{color}
\usepackage{colortbl}
\usepackage{xcolor}
\usepackage{esvect}
\usepackage{diagbox}
\usepackage{bbm}
\usepackage{bm}
\usepackage{enumitem}  
\usepackage[pagewise]{lineno}
\usepackage{placeins}
\usepackage{hyperref}
\usepackage[%
    font={small,sf},
    labelfont=bf,
    format=hang,    
    format=plain,
    margin=0pt,
    width=0.8\textwidth,
]{caption}
\usepackage[list=true]{subcaption}

\newtheorem{theorem}{Theorem}

\newtheorem{claim}{Claim}

\newtheorem{lemma}{Lemma}

\newtheorem{proposition}{Proposition}
\newtheorem{remark}{Remark}

\numberwithin{equation}{section}

\makeatletter
\def\paragraph{\@startsection{paragraph}{4}%
  \z@\z@{-\fontdimen2\font}%
  {\normalfont\bfseries}}
\makeatother

	\renewcommand{\thefootnote}{\arabic{footnote}}

\DeclareMathOperator{\Tr}{Tr}

\DeclareMathOperator{\diag}{diag}

\DeclareMathOperator{\spn}{span}

\newcommand{\calG}{\ensuremath{\mathcal{G}}}
\newcommand{\calI}{\ensuremath{\mathcal{I}}}
\newcommand{\calP}{\ensuremath{\mathcal{P}}}

\newcommand{\calR}{\ensuremath{\mathcal{R}}}
\newcommand{\calS}{\ensuremath{\mathcal{S}}}
\newcommand{\calM}{\ensuremath{\mathcal{M}}}
\newcommand{\calN}{\ensuremath{\mathcal{N}}}

\newcommand{\calK}{\ensuremath{\mathcal{K}}}

\newcommand{\calW}{\ensuremath{\mathcal{W}}}

\newcommand{\calE}{\ensuremath{\mathcal{E}}}

\newcommand{\calO}{\ensuremath{\mathcal{O}}}

\newcommand{\est}[1]{\widehat{#1}}
\newcommand{\hatx}{\est{X}}
\newcommand{\hatxp}{\est{X}_{\calP}}

\newcommand{\norm}[1]{\left\|{#1}\right\|}

\newcommand{\expec}{\ensuremath{\mathbb{E}}}
\newcommand{\matR}{\ensuremath{\mathbb{R}}}
\newcommand{\matN}{\ensuremath{\mathbb{N}}}

\newcommand{\prob}{\ensuremath{\mathbb{P}}}

\definecolor{applegreen}{rgb}{0.55, 0.71, 0.0}

\newcommand{\ones}{\ensuremath{\mathbf{1}}} 

\newcommand{\simp}[1]{\bm{\mathit{\Delta}}_{#1}}

\def\id{\operatorname{Id}}
\def\Xzero{X^{(0)}}
\def\Xone{X^{(1)}}
\def\Xk{X^{(k)}}
\def\Xkone{X^{(k+1)}}
\def\Xkon{X^{(k-1)}}

\def\Xl{X^{(l)}}
\def\greedy{\texttt{GMWM}}

\def\vec{\operatorname{vec}}
\def\simpn2{\bm{\mathit{\Delta}}_{n^2}}
\def\popgradient{\overline{\nabla E}}
\def\gradE{\nabla E}
\def\oN{\overline{N}}
\def\er{Erd\H{o}s-R\'enyi }

\def\hadexp{\exp_{\odot}}
\def\mdgm{\texttt{EMDGM}}
\def\pgdgm{\texttt{PGDGM}}
\def\LEinf{L_{E,\infty}}
\def\LE2{L_{E,2}}
\def\grampa{\texttt{Grampa}}
\def\qpadmm{\texttt{QPADMM}}
\def\blAtilde{\bm{\tilde{A}}}
\def\vtil{{\tilde{v}}^{(1)}}
\def\vtilT{{\tilde{v}}^{(1)T}}
\def\vtilk{{\tilde{v}}^{(1,2)}}
\def\vtilkT{{\tilde{v}}^{(1,2)T}}
\def\Xgrampa{\hat{X}_{\operatorname{grampa}}}
\def\leta{\Lambda_\eta}
\def\ileta{\Lambda^{(-1)}_\eta}
\def\blambda{\bm{\lambda}}
\def\diag{\operatorname{diag}}
\def\gmsimp{{\calG\calM}_{\operatorname{simplex}}}
\def\gmsimpless{{\calG\calM}^{\operatorname{noiseless}}_{\operatorname{simplex}}}
\def\gmK{{\calG\calM}_{\calK}}
\def\prop{\operatorname{suff.cond.}}
\def\propsym{\operatorname{suff.cond.}_{sym}}

\newcommand{\rev}[1]{\textcolor{black}{#1}}
\begin{document}

\title{Graph Matching via convex relaxation to the simplex}

\author{Ernesto Araya\footnotemark[1]\\ \texttt{araya@math.lmu.de}\and Hemant Tyagi \footnotemark[2]\\  \texttt{hemant.tyagi@inria.fr}}

\renewcommand{\thefootnote}{\fnsymbol{footnote}}
\footnotetext[2]{Inria, Univ. Lille, CNRS, UMR 8524 - Laboratoire Paul Painlev\'{e}, F-59000 }
\footnotetext[1]{Ludwig-Maximilians-Universit\"at M\"unchen}

\renewcommand{\thefootnote}{\arabic{footnote}}


\maketitle

\begin{abstract}
This paper addresses the Graph Matching problem, which consists of finding the best possible alignment between two input graphs, and has many applications in computer vision, network deanonymization and protein alignment. A common approach to tackle this problem is through convex relaxations of the NP-hard \emph{Quadratic Assignment Problem} (QAP). 

Here, we introduce a new convex relaxation onto the unit simplex and develop an efficient mirror descent scheme with closed-form iterations for solving this problem. 
Under the correlated Gaussian Wigner model, we show that the simplex relaxation admits a unique solution with high probability. In the noiseless case, this is shown to imply exact recovery of the ground truth permutation. Additionally, we establish a novel sufficiency condition for the input matrix in standard greedy rounding methods, which is less restrictive than the commonly used `diagonal dominance' condition. We use this condition to show exact one-step recovery of the ground truth (holding almost surely) via the mirror descent scheme, in the noiseless setting. We also use this condition to obtain significantly improved conditions for the GRAMPA algorithm \cite{Grampa} in the noiseless setting. 

Our method is evaluated on both synthetic and real data, demonstrating superior statistical performance compared to existing convex relaxation methods with similar computational costs.
\end{abstract}

\section{Introduction}\label{sec:intro}

In the context of the Graph Matching (GM) problem, also referred to as \emph{network alignment}, we consider two input undirected graphs, $G=([n],\calE)$ and $G'=([n],\calE')$, which are assumed to share the same set of vertices, denoted as $[n]=\{1,\cdots,n\}$. The objective is to find a bijective map, represented by a permutation $x:[n]\rightarrow [n]$, between the vertices of $G$ and $G'$, which is intended to align their edges as effectively as possible. The alignment is considered optimal when the largest possible number of edges in $G$ and $G'$ satisfy the relation $\{i,j\}\in \calE\iff \{x(i),x(j)\}\in \calE'$. This problem has numerous applications, such as computer vision \cite{sunfei}, network de-anonymization \cite{nay}, pattern recognition \cite{conte, streib}, protein-protein interactions and computational biology \cite{zasla,singh}, to name a few. For instance, it has been used in computer vision to track a moving object by identifying the correspondence between the different portions in two different frames. 

Mathematically, the Graph Matching problem can be formulated as an instance of the NP-hard \emph{quadratic assignment problem}, as follows
\begin{equation}\label{eq:QAP_non_convex_obj}
  \max_{x\in\calS_n}\sum^n_{i,j=1}A_{ij}B_{x(i)x(j)} \equiv \max_{X\in\calP_n}\langle AX,XB\rangle_F,
\end{equation}
where $A,B\in\matR^{n\times n}$ are the (possibly weighted) adjacency matrices of the pairs of input graphs to be matched and $\calP_n$ (resp. $\calS_n$) is the set of $n\times n$ permutation matrices (resp. the set of permutation maps). In words, a solution of \eqref{eq:QAP_non_convex_obj} will maximize the total edge alignment between $A$ and $B$. For our purposes, it will be more useful to consider the following equivalent formulation,
\begin{equation}\label{eq:QAP_convex_obj}
    \min_{X\in\calP_n}\|AX-XB\|^2_F,
\end{equation}
which has the advantage of possessing a convex objective function that we define by 
\begin{equation*}\label{def:obj_function}
    E(X):=\|AX-XB\|^2_F,
\end{equation*}
where the input graphs $A$ and $B$ will be clear from the context. It will be useful to consider the vector form of \eqref{eq:QAP_convex_obj}, and for that we write with a slight abuse of notation\footnote{Allowing $E$ to take vectors in $\matR^{n^2}$ as arguments, using the canonical matrix-vector identification.}
\begin{equation*}\label{def:obj_function_vector}
    E(\vec(X))=\vec(X)^TH\vec(X),
\end{equation*}
where $H:=(\id\otimes A-B\otimes \id)^2$ and $\vec(X)$ represents the vector form of $X$ obtained by stacking its columns (notations are described in Section \ref{sec:notation}). Despite the convexity of $E(\cdot)$ the problem \eqref{eq:QAP_convex_obj} is NP-Hard in the worst case \cite{QAP}, which can be mainly attributed to its combinatorial constraints, and even finding a constant factor approximation remains a hard problem \cite{QAP_nphard}. On the other hand, notice that when $B={X^*}^T AX^*$ for some $X^*\in \calP_n$, \eqref{eq:QAP_convex_obj} is equivalent to the well-known Graph Isomorphism Problem \cite{babai2018group}.
\paragraph{A general algorithmic strategy}
In spite of its worst-case complexity, several efficient algorithms have been proven to solve the problem \eqref{eq:convex_relax_general} within a specific class of instances. Most of the time, these instances are sampled by a statistical model, which generates matrices $A,B$ according to a ground truth permutation $X^*\in\calP_n$ (we will formally introduce statistical models in Section \ref{sec:models}). The meta-strategy used by most graph matching algorithms consists of two steps that can be summarized as follows.
\begin{enumerate}
   \item \textbf{Similarity finding stage.} We construct a data-dependent similarity matrix $\hat{X}(A,B)\in \matR^{n\times n}$, where we expect that $\hat{X}_{ij}(A,B)$ captures the likelihood of $X^*_{ij}=1$, where $X^*$ is the ground truth permutation. 
    \item \textbf{Rounding stage.} The similarity matrix $\hat{X}$ is then projected onto the set of permutations, which, when minimizing the Frobenius norm, is equivalent to solving a \emph{maximum linear assignment problem}. Since analyzing linear assignment can be quite challenging, many algorithms opt for using a greedy rounding strategy. This approach is simpler to analyze, computationally more efficient, and  often exhibits comparable error rates to linear assignment in experiments.
\end{enumerate}
\paragraph{Convex relaxations of graph matching.}
One common approach for the similarity finding stage consists of defining $\hat{X}(A,B)$ as the solution of a relaxation of \eqref{eq:QAP_convex_obj}, where the set of permutation \rev{matrices} is relaxed to a continuous constraint set. Of special importance is the family of convex relaxations which can be expressed in general as
\begin{equation}\label{eq:convex_relax_general}
    \min_{X\in \calK}\|AX-XB\|^2_F.\tag{\(\gmK\)}
\end{equation}
where $\calK\subseteq \matR^{n\times n}$ is a convex set such that $\calP_n\subseteq \calK$. Some of the choices for the set $\calK$ that have been investigated in the literature include the hyperplane defined by $\ones^TX\ones=n$ \cite{Grampa} and the Birkhoff polytope of doubly stochastic matrices \cite{afla}. The latter can be considered \rev{as} the gold standard, in terms of tightness of convex relaxations. However, its analysis has proven to be challenging, and obtaining theoretical guarantees remains largely open. 
\paragraph{Graph matching on the simplex and mirror descent.}
We consider the set of matrices \[\simpn2:=\{X\in\matR^{n\times n}: \ones^T X\ones=1, X_{ij}\geq 0, \forall i,j\in[n]\},\] which can be identified with the usual unit simplex in $\matR^{n^2}$, in the sense that $X\in \simpn2$ if and only if $\vec(X)$ belongs to the unit simplex (we will refer to $\simpn2$ as the unit simplex in the sequel). By choosing $\calK=\simpn2$ in \eqref{eq:convex_relax_general}, we arrive to the following convex relaxation
\begin{equation}\label{eq:convex_relax_simp}
    \min_{X\in \simpn2}\|AX-XB\|^2_F.\tag{\(\gmsimp\)}
\end{equation}
This formulation is referred to as \emph{graph matching on the simplex}. Despite \eqref{eq:convex_relax_simp} representing a natural problem and being an advancement over other convex relaxations (e.g., it is tighter than the $\ones^TX\ones=n$ constraint in \cite{Grampa}), it has been largely overlooked in the literature. We argue that \eqref{eq:convex_relax_simp} presents at least two advantages over existing work.
\begin{itemize}
    \item \textbf{Regularization through positivity.} 
    In the works \cite{Grampa,Grampa2}, the hyperplane constraint $\calK=\{X\in\matR^{n\times n}: \ones^TX\ones=n\}$
    was considered and an additional regularizing term of the form $\eta\|X\|^2_F$ was added to the objective $E(\cdot)$. By directly incorporating the positivity constraints, we achieve a similar regularization effect without the need for a cumbersome explicit regularizer. This approach helps us avoid the task of parameter selection for such a parameter in practice. We show in experiments that it outperforms \cite{Grampa,Grampa2}.
    \item \textbf{Efficient algorithms.} we can solve \eqref{eq:convex_relax_simp} using iterative first-order methods such as \emph{projected gradient descent} (PGD) and \emph{entropic mirror descent} (EMD). Although the existence of efficient projection algorithms onto $\simpn2$ can make PGD within the simplex more expedited compared to PGD within the Birkhoff polytope, employing the EMD algorithm allows us to completely bypass this projection step, as it possesses an explicit update formula in the form of \emph{multiplicative weights updates} \cite{Bubeck,Beck_Teboulle}.
\end{itemize}
%
%

\paragraph{Contributions. } The primary contributions of this paper are outlined below. To our knowledge, these findings mark the first results for both convex graph matching on the simplex and the mirror descent (MD) algorithm for graph matching.

\begin{itemize}
\item We propose an efficient iterative algorithm based on mirror descent to solve the simplex relaxation of graph matching, which improves by $\calO(\log{n})$ in running time with respect to the more common projected gradient descent method. 

\item We introduce a novel sufficient condition for achieving exact recovery, which is notably less stringent than the commonly used diagonal dominance property. We employ this condition to establish exact one-step recovery (i.e., in one iteration) of the ground truth permutation via EMD, in the noiseless case of the Correlated Gaussian Wigner model (see Section \ref{sec:models}). This result holds almost surely, and for any positive step size in EMD. Furthermore, we use this condition to improve a pivotal lemma from \cite{Grampa}, yielding stronger results with a simpler proof.

\item  In the noiseless case of the Correlated Gaussian Wigner model, we prove that \eqref{eq:convex_relax_simp} has a unique solution with high probability, with the solution corresponding to a scaled version of the ground truth permutation. This, along with \cite[Theorems 9.16 and 9.18]{First_order}, ensures that EMD and PGD converge to a unique point (under certain conditions on step size), which in the noiseless setting leads to exact recovery of the ground truth permutation. 

\item Through extensive experiments, we demonstrate that our proposed algorithm performs well on synthetic data, for several correlated random graph models (described in Section \ref{sec:models}). In addition, we show the applicability of our algorithm to real-world datasets where it is seen to outperform other methods either in speed or accuracy. 
\end{itemize}

\subsection{Related work}\label{sec:related_work}
\paragraph{Convex relaxations.}Although many algorithms based on solving a problem of the form \eqref{eq:convex_relax_general} have been proposed \cite{ganMass,spec_align,Lyzin,bach}, some with very good experimental performance, there is still a lack of theoretical understanding regarding their statistical guarantees. In the case where $\calK$ is the set of doubly stochastic matrices, some early theoretical work, which applies under restrictive conditions on the inputs graphs and mostly in the deterministic setting, can be found in \cite{afla}. 
A spectral estimator dubbed GRAMPA, which is the solution of a regularized version of \eqref{eq:convex_relax_general}, taking $\calK$ to be the hyperplane $\ones^TX\ones=n$, was studied in \cite{Grampa}. They prove that, after greedy rounding, their estimator perfectly recovers the ground truth permutation when the graphs are generated with a correlated Gaussian Wigner model, 
in the regime where the correlation is $\sigma=O\left(1/{\log{n}}\right)$ \rev{(note that $\sigma=0$ means perfect correlation)}. There is still a big gap between these guarantees and the sharp information theoretic threshold for recovery, as discussed below. 

\paragraph{Information-theoretic limits of graph matching.} The necessary and sufficient conditions for correctly estimating the matching between two graphs when they are generated at random, from either the correlated Gaussian Wigner or the correlated \er models (see Section \ref{sec:models}), have been investigated by several authors in \cite{CullKi,HallMass,recons_thr,ganassali21a,ganassali22a,mao2022testing,ganassali2022statistical}. For a correlated Gaussian Wigner model, \rev{in a version\footnote{\rev{ There are two versions of this model, one with $\sigma\in[0,1]$ and another (used in this paper, and also for analyzing GRAMPA in \cite{Grampa}) allowing $\sigma>1$ (see Remark \ref{rem:other_CGW} ).}} with $\sigma\in[0,1]$,} it has been shown in \cite[Thm.1]{recons_thr} that the ground truth permutation $x^*$ can be exactly recovered w.h.p. only when $\sigma^2\leq 1-\frac{(4+\epsilon)\log n}{n}$. When $\sigma^2\geq 1-\frac{(4-\epsilon)\log n}{n}$ no algorithm can even partially recover $x^*$. One natural question pertains to the possibility of closing the gap between the performance bound for GRAMPA as outlined in \cite{Grampa} and the information-theoretic threshold by utilizing convex relaxations.
\paragraph{MD algorithm and Multiplicative Weights Updates.}The MD algorithm, introduced in \cite{Nem_yudin}, has been extensively studied in convex optimization \cite{Beck_Teboulle, Bubeck}, learning theory, and online learning \cite{hazan}. On the other hand, the EMD algorithm, also known as Exponentiated Gradient Descent (EGD) \cite{warmuth} or exponential Multiplicative Weights Updates (MWU) \cite{arora}, has found diverse applications spanning machine learning \cite{MWUneural}, optimization \cite{arora2,arora3,anderson2014}, computer science \cite{arora,Allen-zhu}, and game theory \cite{DASKALAKIS2015}, among others.
Although prior work, such as \cite{Ding_Jordan_MWU} and \cite{MWU_gm_birkhoff}, has explored MWU strategies for graph matching, there are notable distinctions from our research. Firstly, they consider different objective functions and constraint sets, yielding fundamentally different update strategies. Secondly, their approach lacks a principled assessment of robustness against noise due to the absence of statistical generative models for inputs $A$ and $B$. It is worth noting that the same exponential MWU technique we explore here was previously considered in the context of the nonconvex formulation of Graph Isomorphism in \cite{MWU_replicator}. However, it was primarily proposed as a heuristic inspired by the replicator equations from theoretical biology, and no connection with the MD algorithm was established in that work.  

\paragraph{Other algorithms for seedless graph matching.}Although less related to our work, it is worth mentioning that there exist several approaches for seedless graph matching not directly based on convex relaxations. When $\calK$ is the set of orthogonal matrices in \eqref{eq:convex_relax_general}, its solution can be expressed in closed form in terms of the spectrum of the graphs, as was proven in the celebrated work \cite{Spectral_weighted_Ume}. In \cite{MaoRud}, the authors prove that a two-step algorithm can attain exact recovery in the context of the sparse \er graphs, even in the constant correlation regime. For sparse \er graphs, other approaches based on counting combinatorial structures include \cite{Mao_chandeliers,GanassaliM20}. Exploring the possibility of extending those approaches to the dense case remains an open question.

\subsection{Notation.}\label{sec:notation}
We denote $\calP_n$ (resp. $\calS_n$) the set of permutation matrices (resp. permutation maps) of size $n$. We define the vectorization operator $\vec: \matR^{n\times n}\rightarrow \matR^{n^2\times 1}$, which, for a matrix $M$, $\vec{(M)}$ gives the corresponding vector formed by stacking the columns of $M$. For any given matrix $M$, we define the norm $\|M\|_{1,1}:= \sum_{i,j}|M_{ij}|$ and the usual Frobenius norm $\|M\|_F:=(\sum_{i,j}|M_{ij}|^2)^\frac12=\sqrt{\langle M,M\rangle_F}$, where for matrices $M,M'$, $\langle M,M'\rangle_F:=\left\langle \vec(M),\vec(M')\right\rangle$ (here $\langle \cdot,\cdot \rangle$ is the usual inner product in $\matR^{n^2\times 1}$). Similarly, for $M,M'\in \matR^{n\times n}$, we define $M\odot M'$ to be their Hadamard (entrywise) product and $M\otimes M'$ their usual Kronecker product. For a matrix $M$ we define its Hadamard (entrywise) exponential $\hadexp{(M)}$ to be matrix with $(i,j)-$th entry $\exp(M_{ij})$. We use $\ones$ (resp. $J$) to denote the vector (resp. matrix) of all ones in $\matR^n$ (resp. $\matR^{n\times n}$), where the size $n$ will be clear from context. We define $\simp{n}:=\{x\in \matR^n: \ones^Tx = 1\}$ and we identify $\simpn2$ with the set of matrices $M\in\matR^{n\times n}$ such that $\vec(M)\in \simpn2$. We denote $\id$ the identity matrix, where the dimension will be clear from context. We will denote $e_1,\ldots, e_n$ the elements of the canonical basis of $\matR^n$.

\section{Algorithms}\label{sec:algorithms}
%
\subsection{A greedy rounding procedure and a sufficiency lemma}\label{sec:greedy}
We consider the following greedy rounding algorithm, named Greedy Maximum Weighted Matching (GMWM) algorithm in \cite{LubSri,YuXuLin}, which will serve us to accomplish Step 2 in the meta-strategy described earlier. 
The specific version presented below, as Algorithm \ref{alg:gmwm}, has been taken from \cite{ArayaBraunTyagi}. 
\begin{algorithm}
\caption{\texttt{GMWM} (Greedy maximum weight matching)}\label{alg:gmwm}
\begin{algorithmic}[1]
  \Require{A cost matrix $C\in\mathbb{R}^{n\times n}$.}
  \Ensure{A permutation matrix $X\in\calP_n$.}
  \State Select $(i_1,j_1)$ such that $C_{i_1,j_1}$ is the largest entry in $C$ (break ties arbitrarily). Define $C^{(1)}\in\mathbb{R}^{n\times n}$: $C^{(1)}_{ij}=C_{ij}\ones_{i\neq i_1,j\neq j_1}-\infty\cdot\ones_{i=i_1\text{or } j= j_1}$.
  
  \For{$k=2$ to $N$}
        \State  Select $(i_k,j_k)$ such that $C^{(k-1)}_{i_k,j_k}$ is the largest entry in $C^{(k-1)}$ (break ties arbitrarily).
        
        \State Define $C^{(k)}\in\mathbb{R}^{n\times n}$: $C^{(k)}_{ij}=C^{(k-1)}_{ij}\ones_{i\neq i_k,j\neq j_k}-\infty\cdot\ones_{i=i_k\text{or } j= j_k}$.
        
      \EndFor 
      \State Define $X\in \{0,1\}^{n\times n}$: $X_{ij}=\sum^N_{k=1}\ones_{i=i_k,j=j_k}$.
      \State\Return{$X$}
  \end{algorithmic}
\end{algorithm}
To simplify our analysis, we will consider the notion of \emph{algorithmic equivariance}, which, in the context of graph matching, means that if an algorithm outputs an estimator $\hat{X}_{\calP}$ for inputs $A$ and $B$, then the same algorithm outputs $\hat{X}_{\calP} \Pi^T$ when the inputs are $A$ and $\Pi B \Pi^T$ instead, for any $\Pi\in \calP_n$. The following lemma tells us that any algorithm that, first, obtains $\hat{X}$ by solving \eqref{eq:convex_relax_general} for a permutation-invariant set $\calK$ and, second, rounds $\hat{X}$ using Algorithm \ref{alg:gmwm}, is equivariant. We include its proof in Appendix \ref{sec:proofs} for completeness. 
\begin{lemma}\label{lem:equivariance}
    Let $\calK$ be a convex subset of $\matR^{n\times n}$ such that $\calK\Pi^T=\calK$, for any $\Pi\in\calP_n$, and for $A,B\in\matR^{n\times n}$ define $\calS(A,B)$ to be the set of solutions of \eqref{eq:convex_relax_general}. Assume that the set defined by
    \begin{equation*}
        \{Y\in \matR^{n\times n}: Y=\texttt{GMWM}(\hat{X}) \text{ for }\hat{X}\in \calS(A,B)\},
    \end{equation*}
    is a singleton. Then the algorithm that outputs $\hat{X}_\calP=\texttt{GMWM}(\hat{X})$ for any $\hat{X}\in \calS(A,B)$ is equivariant. 
\end{lemma}
The previous lemma implies that we can consider, without loss of generality, the ground truth to be $X^*=\id$ (see \cite[Section 1.2]{Grampa} for additional discussion). This motivates the following deterministic lemma, \rev{which provides sufficient conditions for  \texttt{GMWM} to yield the identity permutation.}
\begin{lemma}[\rev{Sufficient conditions for $\texttt{GMWM}$ returning the identity}]\label{lem:sufficient_prop}
Let $C\in \matR^{n\times n}$ be any matrix. Then the following statements hold. 
\begin{enumerate}[label=(\roman*)]
    \item If $C$ satisfies \begin{equation}\label{eq:property}
    C_{ii}\vee C_{jj}>C_{ij}\vee C_{ji},\ \forall i\neq j\in [n],\tag{\(\prop\)}
    \end{equation}
    then $\greedy(C)=\id$. 
    \item If $C$ is symmetric and 
\begin{equation}\label{eq:property_sym}
    C_{ii}+C_{jj}>2C_{ij},\ \forall i\neq j\in [n],\tag{\(\propsym\)}
\end{equation}
then $\greedy(C)=\id$ .
   \item If $C$ is symmetric and positive definite, we have $\texttt{GMWM}(C)=\id$. 
\end{enumerate} 
In $(i)$ and $(ii)$, if for each $i\in[n]$ the set $\arg\max_{j}C_{ij}$ is a singleton, then the sufficient conditions \eqref{eq:property} and \eqref{eq:property_sym} are also necessary.
\end{lemma}
\begin{proof}
Notice that in the first step of $\texttt{GMWM}$ we select the indices $(i_1,j_1)$ corresponding to the largest entry of $C$. Those indices satisfy $i_1=j_1$, otherwise this would contradict \eqref{eq:property}. Then it suffices to notice that in the subsequent steps the matrix $C^{(k)}$ still satisfy \eqref{eq:property} which implies that $i_k=j_k$ for all $k\in [n]$. This proves $(i)$. When $C$ is symmetric, the property \eqref{eq:property_sym} implies \eqref{eq:property}. This proves $(ii)$. On the other hand, if $C$ is positive definite, then for all $i\neq j\in [n]$ we have $(e_i-e_j)^TC(e_i-e_j)>0$, where $e_k$ is the $k$-the element of the canonical basis of $\matR^n$. From this, we deduce that $C_{ii}+C_{jj}> 2C_{ij}$, which directly implies \eqref{eq:property}, proving $(iii)$. To prove the converse, in the case $(i)$, notice that if $\greedy(C)=\id$ and for each $i\in[n]$ the set $\arg\max_{j}C_{ij}$ is a singleton, then there are no ties in lines $1$ and $3$ of Algorithm \ref{alg:gmwm}, and it is clear that the maximum for each matrix $C^{(k)}$ is attained by some element in its diagonal. This implies \eqref{eq:property}. The proof for the converse of $(ii)$ is analogous. 
\end{proof}
In terms of the general algorithmic strategy sketched in Section \ref{sec:intro}, the main use of Lemma \ref{lem:sufficient_prop} will be the following. If we can prove that $\est{X}$, a solution of the convex relaxation \eqref{eq:convex_relax_general},  satisfies \eqref{eq:property} with certain probability, then Lemma \ref{lem:sufficient_prop} implies that $\est{X}_\calP:=\texttt{GMWM}(\est{X})$ is equal to $X^*=\id$ (that is, we obtain exact recovery), with the same probability.
An interesting point is to understand if the properties \eqref{eq:property} and \eqref{eq:property_sym} are preserved under typical matrix operations such as additions, scaling, etc. In this direction, we have the following result, whose proof is direct. 
\begin{lemma}\label{lem:property_preserving_ops}
    Let $C\in\matR^{n\times n}$ be a matrix satisfying \eqref{eq:property}, then for any $t>0$, $\alpha\in\matR$, the matrices $tC$ and $C+\alpha J$ also satisfy \eqref{eq:property}. Let $C',C''\in \matR^{n\times n}$ be any two symmetric matrices satisfying \eqref{eq:property_sym}, then $tC'$,$C+\alpha J$ and $C'+C''$  satisfy \eqref{eq:property_sym}.
\end{lemma}
In words, the previous lemma states that property \eqref{eq:property} is closed under scaling by a positive factor and a global additive translation, while \eqref{eq:property_sym} is also closed under addition. 
%
\subsection{Mirror descent algorithm for graph matching}
Our proposed method consists of two main steps: first, solving \eqref{eq:convex_relax_simp}, which addresses the graph matching problem on the unit simplex and yields $\hatx$; second, rounding this solution using Algorithm \ref{alg:gmwm} to generate $\hatxp$ an estimate of the ground truth $X^*$.
%
%
To efficiently solve \eqref{eq:convex_relax_simp}, we propose the utilization of the mirror descent (MD) method. This first-order iterative optimization algorithm generalizes projected gradient descent by allowing for non-Euclidean\footnote{Here a non-Euclidean metric means simply a metric different from $\|\cdot\|_2=\sqrt{\langle \cdot,\cdot\rangle}$.} distances \cite{Nem_yudin}. To avoid unnecessary generality, we will focus on two primary examples within the MD family on the unit simplex: Projected Gradient descent (PGD) and Entropic Mirror Descent (EMD). As mentioned ealier, EMD has also received the name of Exponentiated Gradient Descent (EGD) \cite{warmuth} and can be seen as an example of the Multiplicative Weights Update (MWU) \cite{arora} strategy. 
\paragraph{EMD and PGD updates.}Both methods receive an initial point $X^{(0)}$ as input but differ in their approach to performing updates. More specifically, \rev{when applied to the problem \eqref{eq:convex_relax_simp}, which involves minimizing $E(X)=\|AX-XB\|^2_F$ over $\simpn2$, the update process for these methods is as follows}
\begin{align}\label{eq:PDG_update}
    \Xkone &= \mathop{\arg \min}\limits_{X\in\simpn2}\sum^n_{i,j=1}\left[\big(\gamma_k\nabla E(\Xk)_{ij}-\Xk_{ij}\big)X_{ij}+X_{ij}^2\right],\\ \label{eq:MD_update}
    \Xkone &= \mathop{\arg \min}\limits_{X\in\simpn2}  \sum^n_{i,j=1}\left[\big(\gamma_k\nabla E(\Xk)_{ij}-1-\log{\Xk_{ij}}\big)X_{ij}+X_{ij}\log{X_{ij}}\right],
\end{align}
where $(\gamma_k)_{k\geq 0}$ is a sequence of positive learning rates (also called the step-sizes). The derivation of these updates is well-known (refer to \cite[Section 9]{First_order}) and, without entering into details, stems from a broader form of MD, which permits the selection of different (Bregman) distances to project onto the set of constraints. In \eqref{eq:PDG_update}, the Euclidean distance is used as the Bregman distance, and it is easy to see that \eqref{eq:PDG_update} is equivalent to \rev{the more practical expression}
\begin{equation}\label{eq:pgd_dynamic_update}
\Xkone=\operatorname{proj}_{\simpn2}\Big(\Xk-\gamma_k\nabla E(\Xk)\Big),\tag{PGD update}
\end{equation}
where $\operatorname{proj}_{\simpn2}$ corresponds to the Euclidean (Frobenius norm) projection onto $\simpn2$ and $\nabla E(X)=A^2X+XB^2-2AXB$. In the case of \eqref{eq:MD_update}, the Bregman distance employed is the Kullback-Liebler divergence, which gives rise to the term `entropic mirror descent'. Perhaps more importantly, from \eqref{eq:MD_update} one can easily derive (as shown in \cite[Example 3.17]{First_order}) the following simplified form (which corresponds to a MWU rule), 
\begin{equation}\label{eq:md_dynamic_update}
    \Xkone_{ij} = \Xk_{ij}\exp{\left(-\gamma_k\nabla E(\Xk)_{ij}+c\right)}, \tag{EMD update}
\end{equation}
where $c$ is a constant which ensures that $\Xkone\in\simpn2$.
Our proposed algorithm, written in matrix language, is summarized in Algorithm \ref{alg:mirror_simplex}. 
\paragraph{Initialization.}We chose to always start with $X^{(0)}=J/n^2$, the barycenter of the unit simplex, which can be seen as an agnostic initialization, given that the Frobenius distance from $\Xzero$ to any permutation is the same. Additionally, this choice satisfies the known boundedness hypothesis for Kullback-Leibler divergence, with respect to the initial point, that guarantees $\calO{\left(\frac1{\sqrt{N}}\right)}$ convergence to the optimal solution in \eqref{eq:convex_relax_simp} (see \cite[Theorem 9.16]{First_order} and \cite[Example 9.17]{First_order}). 
%

\begin{algorithm}
\caption{Entropic mirror descent for Graph Matching (\texttt{EMDGM})}\label{alg:mirror_simplex}
\begin{algorithmic}[1]
  \Require{$A,B\in\matR^{n\times n}$, $N\in\matN$ and sequence $\{\gamma_k\}^N_{k=0}\in\matR_+$.}
  \Ensure{A permutation $\est{X}_{\calP}\in \calP_n$.}
 \State Define the initial point $X^{(0)}=J/n^2$.
 \State Initialize $X_{\operatorname{best}}\leftarrow \Xzero$
  \For{$k=0$ to $N$}
        \State Compute $N_k=\|X^{(k)}\odot\hadexp{\big(-\gamma_k(A^2X^{(k)}+X^{(k)}B^2-2AX^{(k)}B)\big)}\|_{1,1}$
        \State Update $X^{(k+1)}=X^{(k)}\odot\hadexp{\big(-\gamma_k(A^2X^{(k)}+X^{(k)}B^2-2AX^{(k)}B)\big)}/N_{k}$
        \If{$E(\Xkone)<E(X_{\operatorname{best}})$}
        \State Reassign $X_{\operatorname{best}}\leftarrow \Xkone$
        \EndIf
       \EndFor 
      \State\Return{$\est{X}_{\calP}:=\greedy({X_{\operatorname{best}}})$}
  \end{algorithmic}
\end{algorithm}
%
\begin{remark}[Complexity of Algorithm \ref{alg:mirror_simplex}]
    Each iteration within the for loop of Algorithm \ref{alg:mirror_simplex} has two main steps. First, the gradient computation, which has the same time complexity as matrix multiplication, $\calO(n^\omega)$ where $\omega\leq 2.373$ (as established in \cite{Le_Gall}). Second, the Hadamard product of $\Xk$ and the entrywise exponentiated gradient, which has complexity $n^2$. Referring to \cite{YuXuLin}, it holds that Algorithm \ref{alg:gmwm} has complexity $\calO(n^2)$. Hence, overall, Algorithm \ref{alg:mirror_simplex} has a time complexity of $\calO(N n^\omega)$.
\end{remark}
\subsection{Generative models for correlated random graphs}\label{sec:models}
We now recall the two most popular models for the statistical study of the graph matching problem. These models generate a pair of correlated adjacency matrices $A,B$, based on a ground truth permutation $X^*\in \calP_n$.  Noteworthy fact for these models is that the solution of the QAP formulation \eqref{eq:QAP_non_convex_obj} coincides with the maximum likelihood estimator of $X^*$. This is not true for other models, as shown in \cite{geo_2}, in the case of random graph models with an underlying geometric structure. 
\paragraph{Correlated Gaussian Wigner(CGW) model $\calW(n,\sigma,X^*)$.}The Correlated Gaussian Wigner (CGW) model, introduced in \cite{deg_prof}, defines a distribution on pairs of complete weighted graphs, such that each of them is (marginally) distributed as a random symmetric Gaussian matrix and their correlation is described by a single parameter $\sigma$. More specifically, we will say that $A,B\in\matR^{n\times n}$ follow the CGW model, and write $A,B\sim\calW(n,\sigma,X^*)$ if $B={X^*}^T A{X^*}+\sigma Z$ where $\sigma>0$ and $A,Z$ are i.i.d GOE$(n)$ and $X^*\in\calP_n$. We recall that a $n\times n$ matrix $A$ follows the GOE$(n)$ law if its entries are distributed as  \[A_{ij}\sim\begin{cases}\calN(0,\frac1n)\text{ if }i<j, \\ \calN(0,\frac2n)\text{ if } i= j,
\end{cases}\]
and $A_{ij}=A_{ji}$ for all $i,j\in [n]$.
\rev{
\begin{remark}[Another CGW model]\label{rem:other_CGW}
    There is a different version of the CGW model, which is often used in the literature (see e.g. \cite{recons_thr}), where $B=\sqrt{1-\sigma^2}{X^*}^T A X^*+\sigma Z$, for $\sigma \in [0,1]$. In the model we considered in this paper $\sigma>1$ is allowed. 
\end{remark}
}
%
\paragraph{Correlated \er(CER) model $G(n,\sigma,p,X^*)$.}For $\sigma,p\in[0,1]$, we describe the correlated \er model with latent permutation $X^*\in \calP_n$ by the following two-step sampling process.
\begin{enumerate}
    \item $A$ is sampled according to the classic single graph \er model $G(n,p)$, i.e. for all $i<j$, $A_{ij}$ are independent Bernoulli's random variables with parameter $p$, $A_{ji}=A_{ij}$ and $A_{ii}=0$.
    \item Conditionally on $A$, the entries of $B$ are i.i.d according to the following law %
    \begin{equation}\label{eq: ER_def}
        ({X^*}^T B X^* )_{ij}\sim\begin{cases}
        Bern\left(1-\sigma^2(1-p)\right)\quad \text{if}\quad A_{ij}=1,\\
        Bern\left(\sigma^2 p\right)\quad \text{if } A_{ij}=0,
        \end{cases}
    \end{equation}
    \rev{for $i< j$, and $({X^*}^T B X^* )_{ii}=0$, for all $i$. Notice that $A,B$ are simple graphs (without loops).}
\end{enumerate}
Another equivalent description of this model in the literature involve first sampling an \er  ``mother'' graph, and then obtaining $A,B$ as independent subsamples with certain density parameter (related to $\sigma$). We refer to \cite{PedGloss} for details on this alternative formulation.

\section{Theoretical results}\label{sec:theoretical_aspects}
This section is dedicated to presenting our theoretical results.
In Section \ref{sec:uniqueness}, we prove that \eqref{eq:convex_relax_simp} has a unique solution with high probability, which, in the noiseless setting, corresponds to the scaled ground truth permutation. This, along with \cite[Theorems 9.16 and 9.18]{First_order}, ensures that EMD and PGD converge to a unique point (under certain conditions on $\{\gamma_k\}^{N-1}_{k=0}$ and $N$), which in the noiseless setting leads to exact recovery of the ground truth permutation.
In Section \ref{sec:noiseless_EMD}, we examine the dynamic given by \eqref{eq:md_dynamic_update} in the noiseless setting ($\sigma=0$). We establish, as demonstrated in Theorem \ref{prop:noiseless_dyn_onestep} below, that in this scenario, a single iteration of Algorithm \ref{alg:mirror_simplex} is sufficient to almost surely recover the ground truth permutation for any positive step-size $\gamma_0$. In Section \ref{sec:noiseless_grampa}, we prove a stronger version, as shown in Theorem \ref{thm:strong_noiseless_grampa}, of a critical lemma originally found in \cite{Grampa} (specifically, \cite[Lemma 2.3]{Grampa}) for the GRAMPA algorithm. This lemma forms the cornerstone of GRAMPA's theoretical guarantees for the general noisy setting. 
%
%

\subsection{Uniqueness of the solution of simplex Graph Matching}\label{sec:uniqueness}
In this section, we will prove that under the CGW model, the solution of  \eqref{eq:convex_relax_simp} is unique with high probability. As usual, we assume without loss of generality that $X^*=\id$. 
\rev{For the case $\sigma>0$, we have the following lemma, whose proof is deferred to Appendix \ref{app:proof_strong_convex}.}
\rev{
\begin{lemma}\label{lem:strong_convex}
Let $\sigma>0$, then the objective $\norm{AX - XB}_F^2$ is strongly convex a.s. In particular, the problem \eqref{eq:convex_relax_simp} has a unique solution a.s. 
\end{lemma}
}

%
In the sequel, we focus on the noiseless case ($\sigma=0$) where \eqref{eq:convex_relax_simp} reduces to
\begin{equation}\label{eq:convex_relax_simp_noiseless}
    \min_{X\in \simpn2}\|AX-XA\|^2_F.\tag{\(\gmsimpless\)}
\end{equation}
Notice that under the CGW model, $A\sim GOE(n)$, and it is well-known that  $A$ has distinct eigenvalues almost surely \cite[Theorem 2.5.2]{anderson_guionnet_zeitouni_2009}. In this case, we know that $X=\frac1n\id\in\simpn2$ is a solution of \eqref{eq:convex_relax_simp_noiseless}, which implies that any solution has to satisfy $\|AX-XA\|_F=0$ or, equivalently, $AX-XA=0$. In vector notation, this can be written as
\begin{equation*}
    \vec{(AX-XA)}= H\vec{(X)} = 0,
\end{equation*}
where we recall that $H=(\id\otimes A-A\otimes \id)$. It is easy to see that the null space of $H$ is given by $\spn{\{v_i\otimes v_i\}^n_{i=1}}$, where $v_1,\ldots,v_n$ are the eigenvectors of $A$ \rev{(they form a basis of $\matR^n$, almost surely, given the previous discussion)}. Therefore, all the solutions of \eqref{eq:convex_relax_simp_noiseless} satisfy two conditions: first, they are of the form $X = \sum^n_{i=1}\mu_iv_iv^T_i$ for some reals $\mu_1,\ldots,\mu_n$ (i.e. $X\in\spn\{v_1v^T_1,\ldots,v_n v^T_n\}$), and second, they belong to $\simpn2$. 
\begin{theorem}\label{thm:uniqueness}
    Let $A\sim GOE(n)$, then with probability larger than $1-\frac n{2^{n-1}}$, the problem \eqref{eq:convex_relax_simp_noiseless} has a unique solution. 
\end{theorem}
Before proving Theorem \ref{thm:uniqueness}, we first recall some concepts that will be useful in the proof. The following definitions are taken from \cite[Section 6.2]{horn_johnson}. A matrix $M\in\matR^{n\times n}$ is reducible if there exist $P\in \calP_n$, $r\in[n-1]$ and matrices $M_{11}\in \matR^{r\times r},M_{12}\in\matR^{(n-r)\times r},M_{22}\in\matR^{(n-r)\times (n-r)}$ such that
\begin{equation*}
    PMP^T = \begin{bmatrix}
        M_{11}&M_{12}\\
        0_{(n-r),r}&M_{22}
    \end{bmatrix},
\end{equation*}
where $0_{(n-r),r}$ is the matrix with all zero entries in $\matR^{(n-r)\times r}$.
On the other hand, a square matrix is irreducible if it is not reducible. We will denote $\calI$ the set of irreducible matrices and $\calR$ the set of reducible matrices (their dimension will be clear from context). It is easy to see that if $M$ is symmetric, then
\begin{align}
         M\in \calR\iff &\exists m< n, (B_{kk})^m_{k=1} \text{ with }B_{kk}\in \matR^{n_k\times n_k}  \text{ and }B_{kk}\in\calI\nonumber\\\label{eq:Frob_normal_form}
         &\exists P\in\calP_n\text{ s.t }PMP^T=\operatorname{blockdiag}(B_{11},\ldots,B_{mm}).
\end{align}
This block diagonal representation \eqref{eq:Frob_normal_form} is referred to as the Frobenius (or irreducible) normal form of $M$. We will use repeatedly the fact that if $M\in\calI$ and has nonnegative entries, then the ``strong'' form of the Perron-Frobenius theorem (see \cite[Theorem 8.4.4]{horn_johnson}) applies. As a consequence, the largest modulus eigenvalue of $M\in\calI$ is (algebraically) simple and its associated eigenvector can be chosen with all its entries strictly larger than zero.
\begin{proof}[Proof of Theorem \ref{thm:uniqueness}]
    Let $\calS$ be the (random) set of solutions\footnote{This set corresponds to $\calS(A)$ introduced in Section \ref{sec:algorithms}. Here we drop the dependence on $A$ to ease the notation.} of \eqref{eq:convex_relax_simp_noiseless}. By the remarks at the beginning of this section, we have
    \[
    \calS=\spn{\{v_iv^T_i\}^n_{i=1}}\cap\simpn2,
    \]
    where $v_1,\ldots,v_n$ are the eigenvectors of $A$.
    Note that $\calS$ is nonempty, because clearly $\frac1n\id\in \calS$.
    Our objective is to control the probability of the event 
    \[
    \calE:=\left\{\calS\neq \Big\{\frac1n\id\Big\}\right\}.
    \]
    Consider $\calR_1$ to be the class of reducible matrices with a Frobenius normal form having at least one block of size one, and similarly, $\calR_2$ will denote the class of reducible matrices with a Frobenius normal form with all of its blocks of size two or more (so clearly $\{\calR_1,\calR_2\}$ form a partition of $\calR$).
    With this definition, it is easy to see that
    \begin{equation*}
        \calE=\underbrace{\left\{\calS\cap \calI\neq \emptyset\right\}}_{=:\calE_1}\cup \underbrace{\left\{\calS\cap\calR_1\neq \Big\{\frac1n\id\Big\}\right\}}_{=:\calE_2}\cup\underbrace{\left\{\calS\cap\calR_2\neq \emptyset\right\}}_{=: \calE_3}.
    \end{equation*}
    Here, the events $\calE_1$, $\calE_2$, and $\calE_3$ are clearly disjoint. We will bound their probabilities separately. 
    \paragraph{Bounding $\prob(\calE_1
    )$.} Notice that if there exists a matrix $X$ in $\calS\cap\calI$, then for that $X$, the Perron-Frobenius theorem applies, which justifies the implication
    \begin{align*}
        X\in \calS\cap \calI\implies &X\in\spn{\{v_iv^T_i\}^n_{i=1}}, \text{ and }\\
        &\exists i'\in[n]\text{ s.t. }v_{i'}(k)>0\enskip \forall k\in[n],\text{ or }v_{i'}(k)<0\enskip\forall k\in[n],
    \end{align*}
    where $v_{i'}(k)$ denotes the $k$-th coordinate of $v_{i'}$. 
    From this we deduce that 
    \begin{equation*}
        \calE_1\subseteq\{\exists i' \text{ s.t. } v_{i'}(k)>0,\forall k\in[n]\}\cup \{\exists i' \text{ s.t. } v_{i'}(k)<0,\forall k\in[n]\}.
    \end{equation*}
    We will need the following lemma, whose proof can be found in Appendix \ref{app:proof_prob_unif_sphere}. 
    \begin{lemma}\label{lem:proba_pos_eigenvectors}
    Let $v_1,v_2,\ldots,v_n$ be a set of (not necessarily independent) vectors uniformly distributed on the unit sphere $\mathbb{S}^{n-1}$. Then it holds
    \begin{equation*}
        \prob(\exists i'\in[n]  \text{ s.t. } v_{i'}(k)>0,\forall k\in[n])\leq \frac{n}{2^n}.
    \end{equation*}
    \end{lemma}
    Given that $v_1,\ldots,v_n$ are uniformly distributed in the unit sphere $\mathbb{S}^{n-1}$, we use Lemma \ref{lem:proba_pos_eigenvectors} and the union bound, to conclude that 
    \begin{equation*}\label{eq:boundproba1}
        \prob(\calE_1)\leq \frac{n}{2^{n-1}}. 
    \end{equation*}
    \paragraph{Bounding $\prob(\calE_2)$.} 
    Fix any $X\in\calS\cap\calR_1$. We have that there exists at least one block in the Frobenius normal form of $X$ (c.f. \eqref{eq:Frob_normal_form}) with size one. For simplicity, we can assume w.l.o.g. that its first block is of size one, i.e., $B_{11}\in\matR$. Given that entries of $X$ are nonnegative, it is clear that $B_{11}\geq 0$. Then, the following chain of implications holds (denoting $e_1$ the first element of the canonical basis of $\matR^n$)
    \begin{align}
         PXP^Te_1&=B_{11}e_1\nonumber\\
         \implies XP^Te_1 &= B_{11}P^Te_1\nonumber\\
         \implies\sum^n_{i=1}\mu_iv_i\langle v_i,P^Te_1\rangle &= B_{11}P^Te_1\nonumber\\ \label{eq:mu_D}
         \implies VD\bm{\mu}&=B_{11}P^Te_1,
    \end{align}
 where $V$ has as columns $v_1,\ldots, v_n$, $D$ is a diagonal matrix with entries $D_{ii}=\langle v_i,P^Te_1\rangle$ and $\bm{\mu}=(\mu_1,\ldots,\mu_n)^T$. We have the following simple result. 
 \begin{claim}\label{claim:D_invert}
     D is almost surely invertible.
 \end{claim}
    \begin{proof}
        Clearly $D$ is invertible iff $D_{ii}\neq 0$ for all $i\in [n]$. Note that, there exists $k'\in[n]$ such that $P^Te_1=e_{k'}$. Hence, $D_{ii}=\langle v_i,e_{k'}\rangle$ for some $k'\in [n]$. We claim that 
        \begin{equation}\label{eq:uniform_bound_claim_D}
           \prob\big( \forall k,i\in[n],\enskip \langle v_i,e_k\rangle \neq 0\big)=1.
        \end{equation}
        Indeed, for any given pair $i,k\in[n]$, we have that $\prob(\langle v_i,e_k\rangle=0)=0$ (because $v_i$ is uniformly distributed on $\mathbb{S}^{n-1}$). Then \eqref{eq:uniform_bound_claim_D} follows from the union bound.  
        From \eqref{eq:uniform_bound_claim_D}, we deduce that $D_{ii}\neq 0$ a.s. 
    \end{proof}
    Claim \ref{claim:D_invert} implies that $VD$ is almost surely invertible, which in turn implies that there exists a unique $\bm{\mu}$ that satisfies $VD\bm{\mu}=B_{11}P^Te_1$. On the other hand, it is clear that $\bm{\mu}=B_{11}\ones$ is a solution of \eqref{eq:mu_D}, so it has to be the unique solution. But, if $\bm{\mu}=B_{11}\ones$, then $X= B_{11}\id$ \rev{(recalling that $X = \sum^n_{i=1}\mu_iv_iv^T_i$)}. However, the fact that $X\in\calS$ implies that $B_{11}=\frac1n$. In summary, we have demonstrated that for any $X\in \calS\cap\calR_1$, it holds that $X=\frac1n\id$ almost surely. In other words, with probability one, there are no elements in $\calS\cap\calR_1$ other than $\frac1n\id$. Equivalently, $\prob(\calE_2)=0$.
    \paragraph{Bounding $\prob(\calE_3)$.} Fix any $X\in\calS\cap\calR_2$ (assuming it exists). Then for some $P\in\calP_n$, we have
    \begin{equation}\label{eq:defn_B}
PXP^T=\underbrace{\operatorname{blockdiag}(B_{11},\ldots,B_{mm})}_{=:B},
    \end{equation}
    with all the blocks $B_{kk}\in \matR^{n_k\times n_k}$ having size $n_k\geq 2$, and $B_{kk}\in\calI$. We note that 
    \begin{equation*}
        AX=XA\iff PAP^TB = BPAP^T.
    \end{equation*}
    We will use the following lemma.
    \begin{lemma}\label{lem:B_commutes_Atilde}
    For $A\sim \operatorname{GOE}(n)$ and $B$ as defined in \eqref{eq:defn_B}, it holds 
    \begin{equation*}
        \prob\big(\forall Q\in \calP_n, \enskip QAQ^T B\neq BQAQ^T\big)=1.
    \end{equation*}
    \end{lemma}
    \begin{proof}
        Fix $Q\in \calP_n$. Consider the following block decomposition of $\tilde{A}:=QAQ^T$.
    \begin{equation*}
        \tilde{A}=\begin{bmatrix}
            \blAtilde_{11}& \blAtilde_{12}&\cdots&\blAtilde_{1m}\\
            \blAtilde_{21}& \blAtilde_{22}&\ddots&
            \vdots\\
            \vdots& & \ddots & \\
            \blAtilde_{m1}&\cdots & &\blAtilde_{mm}
        \end{bmatrix},
    \end{equation*}
     where $\blAtilde_{kl}$ belongs to $\matR^{n_k\times n_l}$ (we use bold notation here to distinguish $\blAtilde_{kl}$ from the $(k,l)$ entry of $\tilde{A}$). Observe that $\blAtilde_{kl}=\blAtilde_{lk}^T$. With this, it is easy to verify that
    \begin{equation*}
        \tilde{A}B = B\tilde{A} \iff \blAtilde_{kl}B_{ll}=B_{kk}\blAtilde_{kl}, \enskip \forall k,l\in [m],
    \end{equation*}
    which implies that 
    \begin{equation*}\label{eq:block_commuting1}
        \blAtilde_{lk}\blAtilde_{kl}B_{ll}=B_{ll}\blAtilde_{lk}\blAtilde_{kl},\enskip \forall k,l \in[m] \text{ with }k\neq l,
    \end{equation*}
    and 
    \begin{equation*}\label{eq:block_commuting2}
        \blAtilde_{ll}B_{ll}=B_{ll}\blAtilde_{ll},\enskip \forall l \in[m].
    \end{equation*}
    In the sequel, we will assume w.l.o.g. that $n_1\leq n_2$ and only use the following implication
    \begin{equation}\label{eq:B11_commuting}
        \tilde{A}B = B\tilde{A}\implies \blAtilde_{12}\blAtilde_{21}B_{11}=B_{11}\blAtilde_{12}\blAtilde_{21},\enskip \text{ and }\enskip\blAtilde_{11}B_{11}=B_{11}\blAtilde_{11}.
    \end{equation}
    It can be seen that the following facts hold. 
    \begin{enumerate}[label=(\roman*)]
        \item $\tilde{A}\stackrel{dist.}{=}A\sim \operatorname{GOE}(n)$. In particular, $\blAtilde_{11}\sim \operatorname{GOE}(n_1)$, which implies that $\blAtilde_{11}$ has $n_{1}$ different eigenvalues almost surely (c.f. \cite[Thm. 2.5.2]{anderson_guionnet_zeitouni_2009}). Hence, its eigenvectors, denoted $\vtil_1,\ldots,\vtil_{n_1}$ are uniquely defined up to sign change.
        \item $\blAtilde_{12}\blAtilde_{21}=\blAtilde_{12}\blAtilde_{12}^T\sim \operatorname{Wish}(n_1,n_2)$. Here $M\sim \operatorname{Wish}(n_1,n_2)$ means that $M\in \matR^{n_1\times n_1}$ can be written $M=\tilde{M}\tilde{M}^T$, where $\tilde{M}\in \matR^{n_1\times n_2}$ is a matrix with independent Gaussian entries with variance $\frac1n$. From \cite[Prop. 4.1.3]{anderson_guionnet_zeitouni_2009}, we see that $\blAtilde_{12}\blAtilde_{21}$ has $n_1$ different eigenvalues almost surely. Hence, its eigenvectors, denoted $\vtilk_1,\ldots,\vtilk_{n_1}$ are uniquely determined up to sign change.
        \item $\blAtilde_{12}\blAtilde_{21}$ is independent from $\blAtilde_{11}$.
    \end{enumerate}
    Given that $B_{11}\in \calI$, we deduce, by Perron-Frobenius theorem, that its largest eigenvalue has multiplicity one. On the other hand, by \eqref{eq:B11_commuting} we have that
    \[
    B_{11}\in \spn\Big(\{\vtil_i\vtilT_i\}^{n_1}_{i=1}\Big)\cap \spn\Big(\{\vtilk_i\vtilkT_i\}^{n_1}_{i=1}\Big).
    \]
    In other words, $\{\vtil_1,\ldots,\vtil_{n_1}\}$ and $\{\vtilk_1,\ldots\vtilk_{n_1}\}$ are both basis of eigenvectors for $B_{11}$. Assume w.l.o.g that the largest eigenvalue of $B_{11}$ (of multiplicity one) is associated with $\vtil_1$, in the first basis, and $\vtilk_1$, in the second. This implies that either $\vtil_1=\vtilk_1$ or $\vtil_1=-\vtilk_1$. From facts $(i)$ and $(ii)$, $\vtil_1$ and $\vtilk_1$ are uniformly distributed in $\mathbb{S}^{n_1-1}$. From $(iii)$, $\vtil_1$ and $\vtilk_1$ are independent. We deduce that $\prob(\vtil_{1}=\vtilk_{1})=\prob(\vtil_{1}=-\vtilk_{1})=0$. Summarizing, for our fixed permutation $Q$ we have 
    \begin{equation*}
        \prob(QAQ^TB=BQAQ^T)\leq \prob(\vtil_{1}=\vtilk_{1})+\prob(\vtil_{1}=-\vtilk_{1})=0.
    \end{equation*}
    Then, the stated result follows from a union bound over the (finite) set $\calP_n$.
    \end{proof}
    %
    Using Lemma \ref{lem:B_commutes_Atilde}, we conclude that $\prob(\calE_3)=0$.
\paragraph{Concluding the proof.} We have obtained
\begin{align*}
    \prob(\calE)&=\prob(\calE_1)+\underbrace{\prob(\calE_2)}_{=0}+\underbrace{\prob(\calE_3)}_{=0}\\
    &\leq \frac n{2^{n-1}},
\end{align*}
which concludes the argument. 
\end{proof}

\subsection{Noiseless EMD dynamics}\label{sec:noiseless_EMD}
We now consider the ``noiseless'' version of the GM problem, where $\sigma=0$ or equivalently $A=B$ (given that we assumed $X^*=\id$), which makes the gradient  $\nabla E(X)$ that appears in \eqref{eq:md_dynamic_update} equal to $A^2X+XA^2-2AXA$ \rev{(since in this particular case, the function to be minimized is $E(X)=\|AX-XA\|^2_F$)}. We  consider the CGW model, or equivalently, when $A\sim \operatorname{GOE}(n)$. Interestingly, in this case one iteration of Algorithm \ref{alg:mirror_simplex} suffices to recover the ground truth (i.e., we can fix $N=1$). 
\begin{theorem}\label{prop:noiseless_dyn_onestep}
Let $A\sim \text{GOE}(n)$ and let $\Xone$ be the first iterate of the dynamic defined by \eqref{eq:md_dynamic_update} with initialization $\Xzero = J/n^2$. Then $\greedy(\Xone)=\id$ a.s., for any $\gamma_0>0$.
\end{theorem}
To prove the previous theorem, we will need the following auxiliary result.
\begin{lemma}\label{lem:gradient_property}
Let $A\sim \text{GOE}(n)$. Then for any $i,j\in [n]$ with $i\neq j$, $-\gradE (\Xzero)$ satisfies a.s. the property \eqref{eq:property_sym}, i.e. the following holds almost surely 
\[ \gradE(J)_{ii}+\gradE(J)_{jj}-2\gradE(J)_{ij}<0. 
\]
\end{lemma}
\begin{proof}
We use the decomposition 
\begin{equation*}
    \gradE(J)_{ii}+\gradE(J)_{jj}-2\gradE(J)_{ij}=M_{ii}+M_{jj}-2M_{ij}+M'_{ii}+M'_{jj}-2M'_{ij}
\end{equation*}
where $M:=A^2J+JA^2$ and $M':=-2AXA$. Define $a_k$ as the $k$-th column of $A$, with this we have for $i,j\in [n]$\begin{align*}
(A^2J+J A^2)_{ij}&=\langle a_i,\sum^n_{k=1}a_kJ_{kj}\rangle+\langle \sum^n_{k=1}a_kJ_{ik}, a_j\rangle\\
&=\langle \sum^n_{k=1}a_k,a_i+a_j\rangle,
\end{align*} 
which clearly implies  $M_{ii}+M_{jj}=2M_{ij}$. For the term $M'$ we have \begin{equation*}
    M'_{ii}+M'_{jj}-2M'_{ij}=-\underbrace{\langle (a_i-a_j), J(a_i-a_j)\rangle}_{=(\langle a_i-a_j,\ones\rangle)^2}. 
\end{equation*}
Since $\langle a_i-a_j,\ones\rangle=\sum_{k\notin\{ j,i\}}[A_{ik}-A_{jk}]$, it is clear that $\langle a_i-a_j,\ones\rangle\neq 0$ almost surely.
%
%
%
%
\end{proof}
Observe that in this case, it follows from its definition that $\Xone$ is symmetric. Indeed, it is the Hadamard product of $\Xzero$ and $\hadexp{\left(-\gamma_0\gradE(\Xzero)\right)}$, and both are symmetric matrices. 
\begin{proof}[Proof of Theorem \ref{prop:noiseless_dyn_onestep}]
Take $i,j\in[n]$. From \eqref{eq:md_dynamic_update} and $\Xzero = J/n^2$ we deduce
\begin{align*}
\Xone_{ij}&=N^{-1}_0 e^{-\gamma_0\nabla E\left(\frac{J}{n^2}\right)_{ij}}\\
&=N^{-1}_0 e^{-\frac{\gamma_0}{n^2}(A^2J+JA^2-2AJA)_{ij}}
\end{align*}
where the scalar random variable $N_0=\|e^{-\frac{\gamma_0}{n^2}(A^2J+JA^2-2AJA)_{ij}}\|_{1,1}$ is clearly almost surely strictly positive (we will only need this fact in the sequel). Define the following quantity $q_{ij}=\frac{\sqrt{\Xone_{ii}\Xone_{jj}}}{\Xone_{ij}}$, which is well-defined given that $\Xone_{ij}>0$, almost surely. On the other hand, 
\[q_{ij}=\exp{\left(-\gamma_0( \gradE(J)_{ii}+\gradE(J)_{jj}-2\gradE(J)_{ij})\right)}.\]
By Lemma \ref{lem:gradient_property} we deduce that for $i\neq j$, we have almost surely that $q_{ij}>1$. Which means that $\Xone_{ij}<\sqrt{\Xone_{ii}\Xone_{jj}}$, so we deduce that $\Xone_{ij}<\frac12(\Xone_{ii}+\Xone_{jj})$. Consequently, $\Xone$ satisfies \eqref{eq:property_sym}, and given the symmetry of $\Xone$, we deduce that $\greedy(\Xone)=\id$ almost surely.
\end{proof}
\begin{remark}[The \er case]\label{rem:er_case}
Clearly, if we take an \er graph $A\sim G(n,p)$, the argument in Lemma \ref{lem:gradient_property} will not work, since there are, with probability one, at least two vertices (say $i,j\in [n]$) with the same degree, hence $\langle a_i-a_j,\ones\rangle=0$. Using the degree profile distribution of \er graphs, which is a well-studied problem \cite{BOLLOBAS1980,BOLLOBAS1981}, one can give a bound on the number of vertices that will be correctly assigned by Algorithm \ref{alg:mirror_simplex} with $N=1$. We remark that for the graph isomorphism problem, more elaborated ``signatures'' using the degrees have been used with success under the \er model \cite{RG_Isom_BES,Improv_RG_Isom,dai_cullina,deg_prof}.  
\end{remark}
\rev{
\begin{remark}[Diagonal dominance versus sufficiency property in Lemma \ref{lem:sufficient_prop}] Given a similarity matrix $\est{X}$, several existing theoretical results for exact recovery (see \cite{Grampa,YuXuLin,MaoRud,ArayaBraunTyagi}) use the following sufficient condition for ensuring $\greedy(\est{X})=\id$;
\begin{equation}\label{eq:diago_dom}
    \est{X}_{ii}>\max_{j\neq i}\est{X}_{ij},\ \forall i\in[n],\tag{\(\operatorname{diag.dom.}\)}
\end{equation}
which is known as \emph{diagonal dominance}. It is clear that \eqref{eq:diago_dom} implies \eqref{eq:property} (which is equivalent to \eqref{eq:property_sym} in the symmetric case). The approach of using \eqref{eq:diago_dom} instead of \eqref{eq:property_sym} to prove Theorem \ref{prop:noiseless_dyn_onestep} fails. Indeed, when $\est{X}=\Xone$ as in Theorem \ref{prop:noiseless_dyn_onestep}, showing \eqref{eq:diago_dom} is equivalent to showing $ \gradE(J)_{ii}<\gradE(J)_{ij}$, for all $j\neq i $. A simple calculation shows that 
\begin{align*}
    \expec[\gradE(J)_{ii}]&=-1+o(1),\enskip \expec[\gradE(J)_{ij}]=o(1)\enskip \forall i \neq j, \\
    \text{and} \ \operatorname{Var}[\gradE(J)_{ij}] &=3-o(1) \enskip \forall i, j.
\end{align*}
Then using Chebyshev's inequality, we deduce that the event $\gradE(J)_{ii}\geq \gradE(J)_{ij}$ occurs with at least a constant probability, for any given $i \neq j$. Hence it follows that \eqref{eq:diago_dom} will fail to hold with at least a constant probability.
\end{remark}
}
\subsection{Alternative analysis for GRAMPA in noiseless setting}\label{sec:noiseless_grampa}
To further illustrate the advantage of using \eqref{eq:property_sym} over \eqref{eq:diago_dom}, we will prove a stronger version of \cite[Lemma 2.3]{Grampa} which  provides guarantees for exact recovery (in the noiseless case) for the GRAMPA algorithm that operates as follows. \rev{First, it computes the spectral decomposition of $A = \sum^n_{i=1}\lambda_iv_iv_i^{T}$. Second, for a regularization parameter $\eta$, it computes the similarity matrix
\begin{equation}\label{eq:grampa_similarity}
\Xgrampa := \sum^n_{i,j=1}\frac{1}{\eta^2+(\lambda_i-\lambda_j)^2}v_iv_i^TJv_jv_j^T.
\end{equation}
%
%
Finally, the algorithm rounds $\Xgrampa$ using any rounding procedure (e.g., linear assignment, greedy rounding, etc.).}
For the result in \cite[Lemma 2.3]{Grampa}, however, conditions are derived on $\eta$ which ensure that $\Xgrampa$ is diagonally-dominant\footnote{In fact, they consider an even stronger version of diagonal-dominance than \eqref{eq:diago_dom}.}, with high probability. 

In the following theorem, we show for the noiseless setting that exact recovery is in fact possible \emph{for any} choice of $\eta$.
\begin{theorem}\label{thm:strong_noiseless_grampa}
    Let $A$ be a $GOE(n)$ distributed matrix with spectral expansion $A=\sum^n_{i=1}\lambda_iv_iv_i^T$. Then for any $\eta$, the matrix $\Xgrampa$, defined in \eqref{eq:grampa_similarity}, satisfies 
    \begin{equation*}
        \greedy(\Xgrampa)=\id, 
    \end{equation*}
    almost surely.
\end{theorem}
\begin{proof}
    Define the matrix $\Lambda_\eta$, with entries
    \begin{equation*}
        (\Lambda_\eta)_{ij}=(\lambda_i-\lambda_j)^2+\eta^2.
    \end{equation*}
    Observe that $\Xgrampa$ can be rewritten as 
    \begin{equation}\label{eq:grampa_similarity2}
        \Xgrampa = V\left(\ileta\odot (V^TJV)\right)V^T,
    \end{equation}
    where $\ileta$ denotes the entrywise inverse of $\leta$ and $V$ is the matrix with columns $v_1,\ldots,v_n$. Recall that a matrix $M\in \matR^{n\times n}$ is conditionally p.s.d (resp. conditionally p.d) if, for all $v\neq0$ such that $v^T\ones=0$,  we have $v^TMv\geq0$ (resp. $v^TMv>0$); see e.g., \cite[Section 6.3]{Topics_matrix_analysis}. We now have the following claim.
   \begin{claim}\label{claim:grampa_1}
        $-\Lambda_\eta$ is conditionally p.s.d for all $\eta\in\matR$.
    \end{claim}
    \begin{proof}
        We will first prove that $-\Lambda_0$ is conditionally p.s.d. For that, notice that 
        \[
        \Lambda_0 = (\blambda\odot\blambda)\ones^T+\ones(\blambda\odot\blambda)^T-2\blambda\blambda^T,
        \]
        where $\blambda=(\lambda_1,\ldots,\lambda_n)^T$. Take $v\in\matR^n$ such that $v\neq 0$ and $v^T\ones=0$. Then 
        \begin{align*}
            v^T(-\Lambda_0)v&=2(v^T\blambda)^2-\underbrace{v^T(\blambda\odot\blambda)\ones^Tv+v^T\ones(\blambda\odot\blambda)v}_{=0}\\
            &\geq 0.
        \end{align*}
        On the other hand, it is evident that $\leta=\Lambda_0+\eta^2J$. Since adding a scalar times the all ones matrix preserves the conditional p.s.d property, we deduce that $-\leta$ is conditionally p.s.d for any $\eta$.  
    \end{proof}
    Claim \ref{claim:grampa_1} is useful in light of Theorem \ref{thm:condp_to_exppd} (see Appendix \ref{app:psd_facts}), which says that the entrywise exponential of a conditionally p.s.d matrix is p.s.d. We will use this to prove the following claim. 
    \begin{claim}\label{claim:grampa_2}
        $\ileta$ is almost surely p.d for all $\eta \in \matR$.
    \end{claim}
    \begin{proof}
        We will first prove that $\hadexp(-t\leta)$ is p.d, for all $t>0$. For that, using Claim \ref{claim:grampa_1} 
        we obtain that $-t\leta$ is conditionally p.s.d. for all $t\geq 0$. By Theorem \ref{thm:condp_to_exppd} part $(iii)$, $\hadexp(-t\leta)$ is p.d. if and only if $-t\leta$ satisfies \eqref{eq:property_sym}. It is clear that for all $i\neq j$ and $t>0$
        \[
        \underbrace{-t(\leta)_{ii}-t(\leta)_{jj}}_{=-2t\eta^2}>\underbrace{-2t(\leta)_{ij}}_{=-2t(\lambda_i-\lambda_j)^2-2t\eta^2} \enskip \text{almost surely},
        \]
        since $\lambda_i \neq \lambda_j$ almost surely \cite[Theorem 2.5.2]{anderson_guionnet_zeitouni_2009}. Hence, $\hadexp(-t\leta)$ is almost surely p.d. for all $t>0$. Given that $\frac1s=\int^\infty_0e^{-st}dt$, we have the representation 
        \begin{equation*}
            \ileta=\int^\infty_0\hadexp(-t\leta)dt,
        \end{equation*}
        where the integral of a matrix is understood entrywise. \rev{On the other hand, the integral of positive definite matrices is positive definite, as proven in Lemma \ref{lem:integral_psd} in Appendix \ref{app:psd_facts}.}
        Given that $\hadexp(-t\leta)$ is almost surely p.d. for all $t>0$, we conclude that $\ileta$ is almost surely p.d.
    \end{proof}
    From Claim \ref{claim:grampa_2}, $\ileta$ is almost surely p.d and, on the other hand, it is evident that $V^TJV$ is p.s.d. In addition, the diagonal entries of $V^TJV$ (observe that $(V^TJV)_{ii}=(v^T_i\ones)^2$) are almost surely different from zero, since $v_i$'s are uniformly distributed on the sphere (see e.g., \cite{anderson_guionnet_zeitouni_2009}). \rev{In Lemma \ref{lem:hadprod_pd_psd} in Appendix \ref{app:psd_facts} we prove that the Hadamard product of a p.d matrix with a p.s.d matrix is p.d. Using this result
    we obtain that $\ileta\odot(V^TJV)$ is almost surely p.d.} From \eqref{eq:grampa_similarity2} it is easy to see that $\Xgrampa$ is almost surely p.d. Combining this with Lemma \ref{lem:sufficient_prop} part $(iii)$, the conclusion follows. 
\end{proof}
\begin{remark}[Comparison with Lemma 2.3 in \cite{Grampa}]\label{rem:comp_lem_grampa}
    Theorem \ref{thm:strong_noiseless_grampa} is stronger than \cite[Lemma 2.3]{Grampa} in two aspects. First, it provides an almost sure guarantee instead of the high probability bound in \cite[Lemma 2.3]{Grampa}. Second, it holds for any $\eta\in\matR$, while in \cite[Lemma 2.3]{Grampa} the result requires that $\frac1{n^{0.1}}<|\eta|<\frac c{\log n}$ for some constant $c$. In addition, our approach also leads to a shorter and arguably simpler proof. 
\end{remark}
\begin{remark}[Extension to the Erd\H{o}s-R\'enyi case]\label{rem:extension_ER_lemma_grampa}
    Notice that in the proof of Theorem \ref{thm:strong_noiseless_grampa}, the assumption that $A\sim\operatorname{GOE}(n)$ is crucial at two key points. First, to establish Claim \ref{claim:grampa_2}, it is necessary that $A$ almost surely possesses a simple spectrum. Second, we rely on this assumption to demonstrate that the quantities $v_i^T\ones$ are almost surely non-zero for $i\in[n]$. In light of \cite[Corollary 1.4]{Tao2017}, we believe that it should be relatively straightforward to extend Theorem \ref{thm:strong_noiseless_grampa} to the case where $A\sim G(n,p)$, at least in a high probability sense.
\end{remark}

\section{Numerical experiments}\label{sec:experiments}
We now provide numerical experiments using both synthetic and real data to assess the performance of Algorithm \ref{alg:mirror_simplex}. We evaluate its performance along with other competing graph matching algorithms, considering both accuracy and running time as key metrics.
In graph matching, the customary measure of accuracy is the overlap with the ground truth $X^*\in\calP_n$, also known as the \emph{recovery fraction}. This measure, denoted as $\operatorname{overlap}(X^*, \est{X}_\calP)$, is defined as: 
\begin{equation*}\label{eq:overlap}
    \operatorname{overlap}(X^*,\est{X}_\calP):=\frac1n\langle X^*,\est{X}_\calP\rangle_F,
\end{equation*}
where $\est{X}_\calP$ represents the output of any graph matching algorithm, and we recall that $\langle X^*,\est{X}_\calP$ denotes the Frobenius inner product between $X^*$ and $\est{X}_\calP$.

All the experiments were conducted using MATLAB R2021a (MathWorks Inc., Natick, MA) on an Apple M1 machine with 8 cores, 3.2 GHz clock speed, and 16GB of RAM. The reported running times are based on this setup. The code is available
at \url{https://github.com/ErnestoArayaV/Graph-Matching-Convex-Relaxations}.
\subsection{Synthetic data setup}\label{sec:exp_synthetic_data}
For the synthetic data experiments, we follow this general setup: we generate matrices $A$ and $B$ using the CER and CGW models described in Section \ref{sec:models}, with $X^*$ chosen uniformly at random from $\calP_n$. We then apply various graph matching methods based on convex relaxations, using $A$ and $B$ as inputs.
%
%
%
\paragraph{Standardization of the inputs.} To ease the comparison between different models, we will consider their standardized version. Following \cite{Grampa}, we consider the following standardized version of the CER model described in Sec. \ref{sec:models}
\[\tilde{A} = \frac{\left(A-\expec(A)\right)}{p(1-p)n},\enskip \tilde{B} = \frac{\left(B-\expec(B)\right)}{p(1-p)n},\]
where $A,B\sim G(n,\sigma,p,X^*)$. This can be regarded as a simple preprocessing step in the pipeline of the three graph matching algorithms that will be experimentally tested here. 
\paragraph{Step-size selection.}An obvious question for any gradient descent-like method is the selection of a step-size updating rule. We investigate the following popular strategies for EMD and PGD. 
\begin{itemize}
    \item \textbf{Fixed step-sizes.} In \cite[Theorem 9.16 and Example 9.17]{First_order} it is proven that MD converges in $\calO\left(\frac1{\sqrt{N}}\right)$ iterations, under the following constant step-size rule
    \begin{equation}\label{eq:fixed_iter_md}
        \gamma^{\operatorname{md}}_k = \frac{\sqrt{2\log{n}}L_{E,\infty}}{\sqrt{N+1}},\ \ \forall k\in [N],
    \end{equation}
    where $\LEinf:=\max\limits_{X\in\simpn2}\|\gradE{(X)}\|_{\infty}$ and for $v\in \matR^d$, with $d\in\matN$, $\|v\|_\infty:=\max_{j\in\matN}|v_j|$. Given that $\LEinf$ is, in general, not observed, for each $k$ we replace it (as a heuristic) with the following observed quantity
    \begin{equation*}\label{eq:grad_norm_inf_emp}
        \hat{L}^{(k)}_{E,\infty} = \max\limits_{X\in \{\Xzero,\ldots,\Xk\}}\|\gradE{(X)}\|_{\infty}.
    \end{equation*}
    Notice that, contrary to \eqref{eq:fixed_iter_md}, the rule obtained by replacing $L_{E,\infty}$ by the empirical quantity $\hat{L}^{(k)}_{E,\infty}$ is not a fixed step strategy.
    Analogously, for the PGD algorithm, based on the same theoretical result we define the following update rule
    \begin{equation}\label{eq:fixed_iter_pgd}
        \gamma^{\operatorname{pgd}}_k = \frac{\sqrt{2}\hat{L}^{(k)}_{E,2}}{\sqrt{N+1}}, \ \ \forall k\in [N],
    \end{equation}
    where 
    \begin{equation*}\label{eq:grad_norm_2_emp}
    \hat{L}^{(k)}_{E,2} = \max\limits_{X\in \{\Xzero,\ldots,\Xk\}}\|\gradE{(X)}\|_{F}.
    \end{equation*}
    \item \textbf{Dynamic step-sizes.} The $\calO\left(\frac1{\sqrt{N}}\right)$ convergence of EMD and PGD has also been established, see \cite[Theorem 9.18]{First_order}, under the following alternative strategy 
    \begin{equation}\label{eq:dynamic_step_md}
         \gamma^{\operatorname{md}}_k=
          \begin{cases}
        \frac{\sqrt{2}}{\|\gradE(\Xk)\|_{\infty}\sqrt{k+1}}\text{ if }\gradE(\Xk)\neq 0 \\
        0\qquad \qquad \qquad\text{if }\gradE(\Xk)= 0,
        \end{cases}
    \end{equation}
    %
    Notice that here $\gamma^{\operatorname{md}}_k$ can change from iteration to iteration, as it is adapting to the current gradient information. In the case of PGD the analogous rule is given by 
    \begin{equation*}
        \gamma^{\operatorname{pgd}}_k=\begin{cases}
        \frac{\sqrt{2}}{\|\gradE(\Xk)\|_{F}\sqrt{k+1}}\text{ if }\gradE(\Xk)\neq 0 \\
        0\qquad \qquad \qquad\text{ if }\gradE(\Xk)= 0.
        \end{cases}
    \end{equation*}
    In the case that $\gradE(\Xk)=0$, we have chosen to impute $\gamma^{\operatorname{md}}_k$ and $\gamma^{\operatorname{pgd}}_k$ with a zero value, but this choice is arbitrary and has no effect in the value of the iterates.
    %
    \item \textbf{Other heuristic rules.} We also consider certain purely heuristic rules, which are not based on theoretical convergence results. In our experiments, these rules sometimes outperform the theoretically established rules described above. For PGD we consider 
    \begin{equation}\label{eq:heuristic_pgd_1}
        \gamma_k = \theta \frac{\|\gradE(\Xk)\|_F^2}{\|E(\Xk)\|_F^2},
    \end{equation}
    where $\theta$ is a positive constant, which is either fixed to $1$ (default value) or determined by line search. 
\end{itemize} 
\subsubsection{Comparison with other convex optimization-based methods }\label{sec:exp_comparison_convex}
We compare some of the state-of-art seedless graph matching methods based on convex relaxations with our proposed method, Algorithm \ref{alg:mirror_simplex}, using a fixed number of iterations $N=125$. In terms of efficient methods for convex graph matching, two of the better performing ones are \texttt{Grampa} \cite{Grampa} and \texttt{QPADMM}, the latter of which solves \eqref{eq:convex_relax_general} with $\calK$ the set of doubly stochastic matrices, using the method of Alternating Direction Method of Multipliers (ADMM). The \texttt{QPADMM} algorithm has also been considered in the numerical experiments sections in \cite{deg_prof, Grampa}, under the name \texttt{QS-DS}. We recall that $\grampa$ is a regularized spectral method, which adds the term $\eta\|X\|^2_F$ to the objective function $E(\cdot)$ in \eqref{eq:convex_relax_general} with $\calK=\{\ones^TX\ones=n\}$. Unless otherwise specified, we adopt the recommended value of $\eta=0.2$ (see \cite[Section 4]{Grampa}). In our experiments, we observed that variations in this parameter have minimal impact on the overall performance of $\grampa$, which aligns with findings reported in \cite{Grampa}. We also include the classic \texttt{Umeyama} algorithm, which was originally introduced in \cite{Spectral_weighted_Ume}, and corresponds to an unregularized spectral method. 
%
%
\begin{figure}[h]
    \centering
        \begin{subfigure}{.5\textwidth}
        \centering
        \includegraphics[scale=0.22]{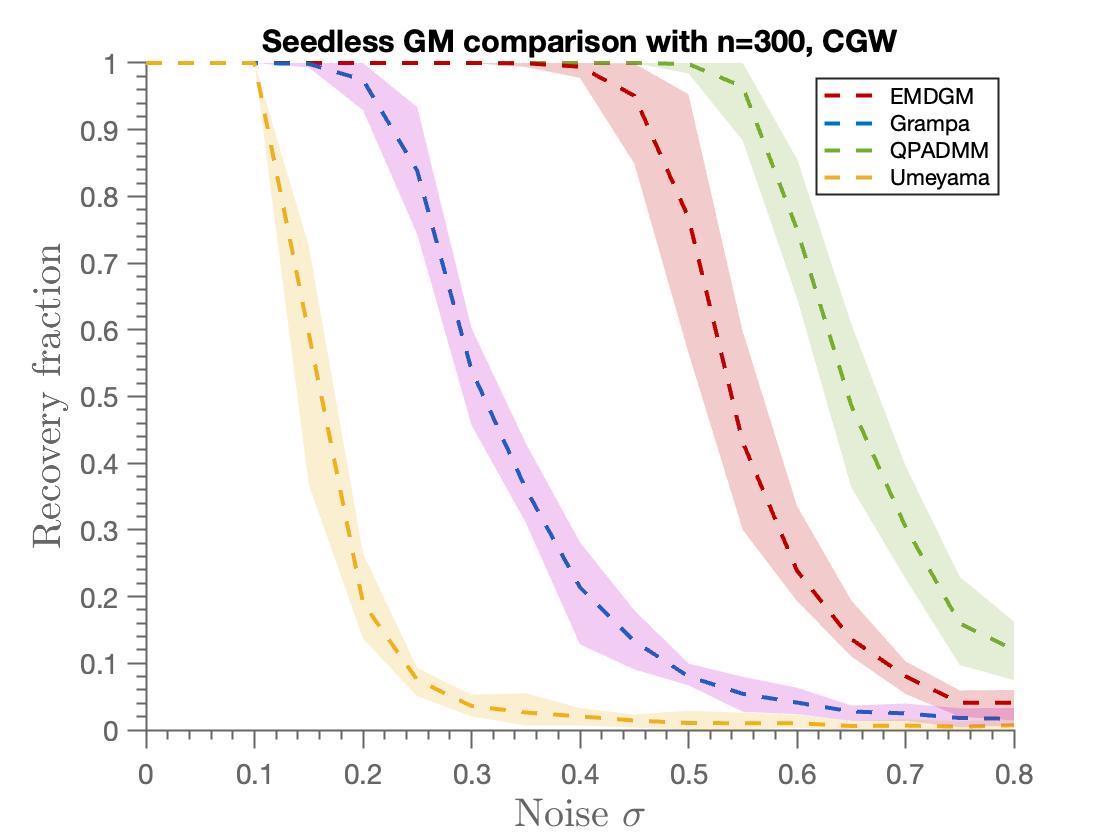}
    \caption{CGW model.}
    \label{fig:comp_convex_a}
    \end{subfigure}%
    \begin{subfigure}{.5\textwidth}
         \centering
         \includegraphics[scale=0.22]{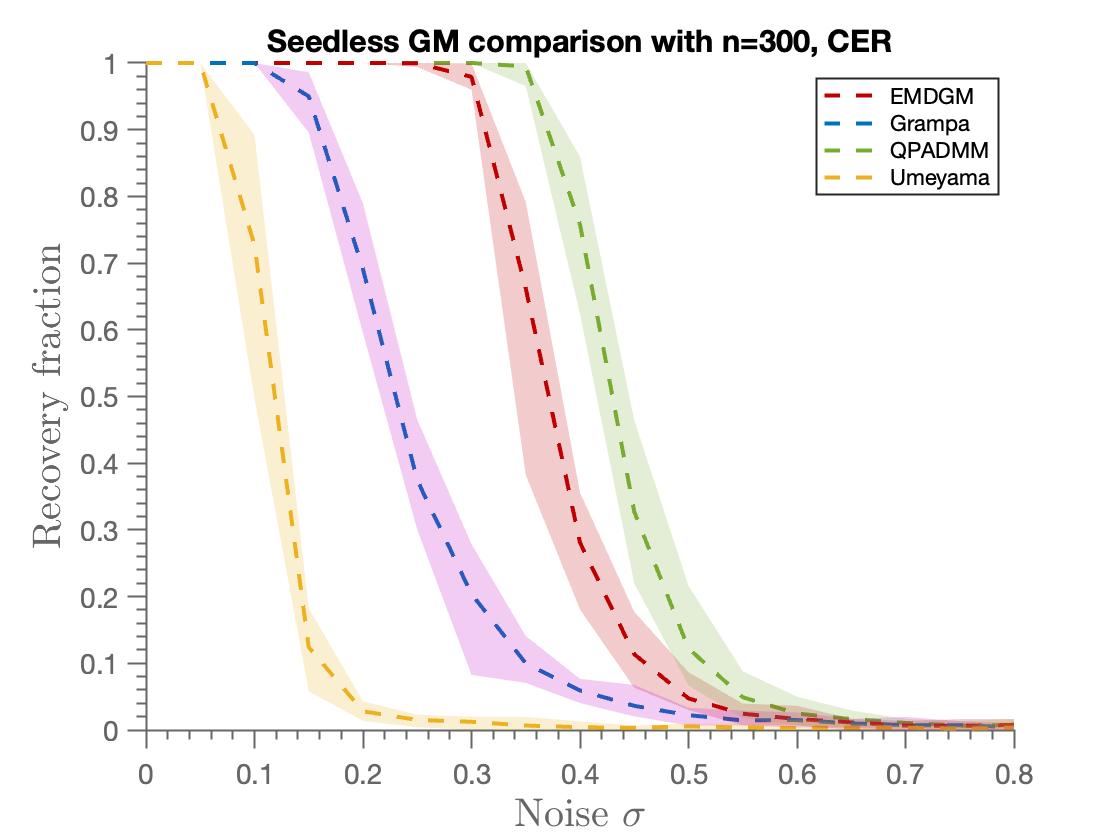}
    \caption{CER model with $p=\frac12$.}
    \label{fig:comp_convex_b}
    \end{subfigure}
    \caption{Comparison of the performance of seedless convex methods for CGW (Fig. \ref{fig:comp_convex_a}) and CER (Fig. \ref{fig:comp_convex_b}) models. We plot the average overlap with the ground truth (recovery fraction) over $15$ Monte Carlo runs for graphs of size $n=300$. We used $N=125$ iterations in $\mdgm$ and in $\grampa$ we use the regularization parameter $\eta=0.2$. For $\mdgm$ we used the dynamic step-size rule \eqref{eq:dynamic_step_md}.}
    \label{fig:compar_md_grampa_qpadmm}
\end{figure}

\paragraph{Accuracy comparison.}Figure \ref{fig:compar_md_grampa_qpadmm} illustrates that, in terms of accuracy, the performance of $\mdgm$ sits between $\grampa$ and $\qpadmm$, while \texttt{Umeyama} is outperformed by all three methods. While this result is not unexpected, considering the well-defined hierarchy of relaxations from the hyperplane to the simplex and, finally, to the Birkhoff polytope, these experiments aim to provide a clearer understanding of the gap between them. In this regard, taking the case of the CGW model in Fig.\ref{fig:comp_convex_a}, we see that $\mdgm$ achieves exact recovery for noise levels as high as $\sigma\approx 0.5$, while $\grampa$ can recover only up to $\sigma\approx 0.25$. On the other hand, $\qpadmm$ outperforms all the rest, with a recovery of up to $\sigma\approx 0.6$. These experiments show that $\mdgm$ (with $N=125$) is closer, in terms of accuracy, to the tightest relaxation of $\qpadmm$. As we will see in the sequel, the performance of $\mdgm$ is highly correlated with the number of iterations. 
In the case of the CER model, the gap between methods reduce and even $\qpadmm$ can recover only up to $\sigma\approx 0.4$. This confirms that matching CER graphs is more challenging for comparable levels of noise. One possible explanation for this fact is that the (continuous) weights in the CGW model help to discriminate the correct from incorrect matches (``breaking the symmetry'') more effectively than in the binary weights case, where one relies solely on the graph structure. From the theoretical point of view, we saw a manifestation of this phenomenon in Section \ref{sec:theoretical_aspects}, when trying to extend Theorem \ref{prop:pop_dyn_onestep} to the CER case (see Remark \ref{rem:er_case}). Overall, these experiments show that $\mdgm$ is closer, in terms of accuracy, to the tightest relaxation of $\qpadmm$.  
\rev{
\begin{remark}\label{rem:exp_other_algos}
    Although it is not the main focus of this section, it is worth noting that in \cite[Section 4]{Grampa}, \texttt{Grampa} has been empirically demonstrated to outperform a range of other algorithms (not necessarily based on convex optimization) in terms of accuracy, including \texttt{TopEigenVec} \cite{Grampa}, \texttt{IsoRank} \cite{singh,isorank_2}, \texttt{EigenAlign} \cite{spec_align}, \texttt{LowRankAlign} \cite{spec_align}, and \texttt{DegreeProfile} \cite{deg_prof}. Thus we avoid comparing our \mdgm\ algorithm with these other methods.
\end{remark}
}
\paragraph{Running time comparison.} In Table \ref{tab:comp_convex_run_time}, we summarize the average running time and standard deviation over the $15$ Monte Carlo runs, for each algorithm. In these experiments, we observe that both $\grampa$ and $\texttt{Umeyama}$ are of comparable speed, while $\mdgm$ operates about eight times slower than them. Conversely, $\qpadmm$ is the slowest, being about a thousand times slower than $\grampa$. In the next subsection, we will explore how it is possible to achieve good accuracy with $\mdgm$ using fewer iterations, thus speeding up the overall process.
\begin{table}[h]
    \centering
    \begin{tabular}{l||c|c|c|c}    
       &\texttt{Umeyama}  & $\mdgm$ &$\grampa$ &$\qpadmm$ \\
       \hline
      CGW&  0.0356(0.0211) & 0.8609(0.0282) & 0.0490(0.0048)&34.2666(1.1065)\\
      CER&  0.0382(0.015) & 0.9226(0.0168) & 0.0432(0.0063)& 29.1306(1.1302)\\
    \end{tabular}
    \caption{Average (over $15$ Monte Carlos) running time in seconds (standard deviation in parentheses).}
    \label{tab:comp_convex_run_time}
\end{table}
\subsubsection{Comparison with PGD}\label{sec:exp_comparison_firstorder}
As mentioned earlier, the EMD algorithm is one of several algorithmic alternatives to solve \eqref{eq:convex_relax_simp}. In this section, we have two objectives: Firstly, to compare it with other algorithms that achieve the same task, and secondly, to evaluate its performance with different values of $N$. For runtime efficiency, we confine our comparison to the category of first-order methods, among which PGD is arguably the most popular one. We will call $\texttt{PGDGM}$ the algorithm that solves \eqref{eq:convex_relax_simp} using a gradient descent approach, receiving the same inputs and following the same pipeline as Algorithm \ref{alg:mirror_simplex}, but changing the update rule from \eqref{eq:md_dynamic_update} to \eqref{eq:pgd_dynamic_update}. Given the existing theoretical convergence guarantees for MD and PGD, see \cite[Sections 8 and 9]{First_order}, we know that there exists rate sequences $(\gamma_k)_{k\geq 0}$ such that both methods will converge to a minimizer of \eqref{eq:convex_relax_simp}. Our objective here is to provide a practical comparison of both algorithms within a unified setup, focusing on the number of iterations and the selection of step-size rules. In this section, we will demonstrate that the overall results of both approaches often diverge on synthetic data (a pattern consistent with real data, as we demonstrate later in Section \ref{sec:exp_real_data} below).
\begin{figure}
    \centering
    \begin{subfigure}{.5\textwidth}
        \centering
        \includegraphics[scale=0.22]{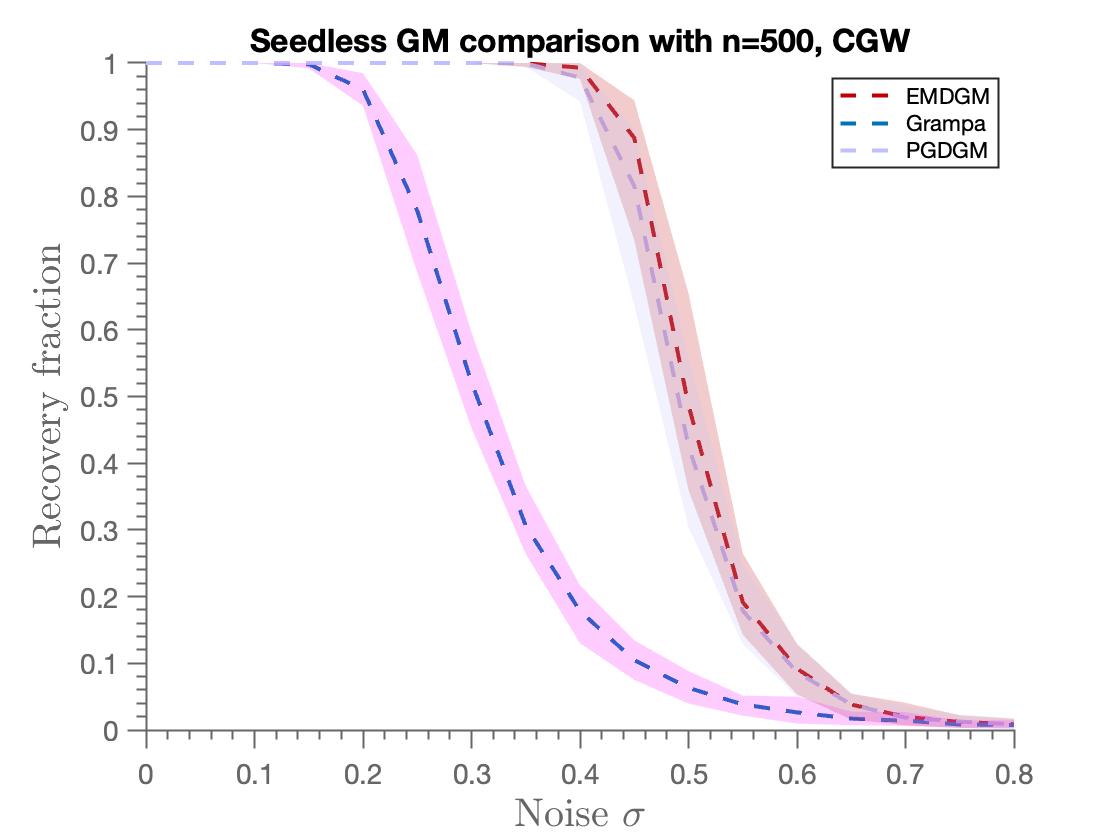}
    \caption{$N=125$ in the CGW model.}
    \label{fig:pgd_md_a}
    \end{subfigure}%
    \begin{subfigure}{.5\textwidth}
         \centering
         \includegraphics[scale=0.22]{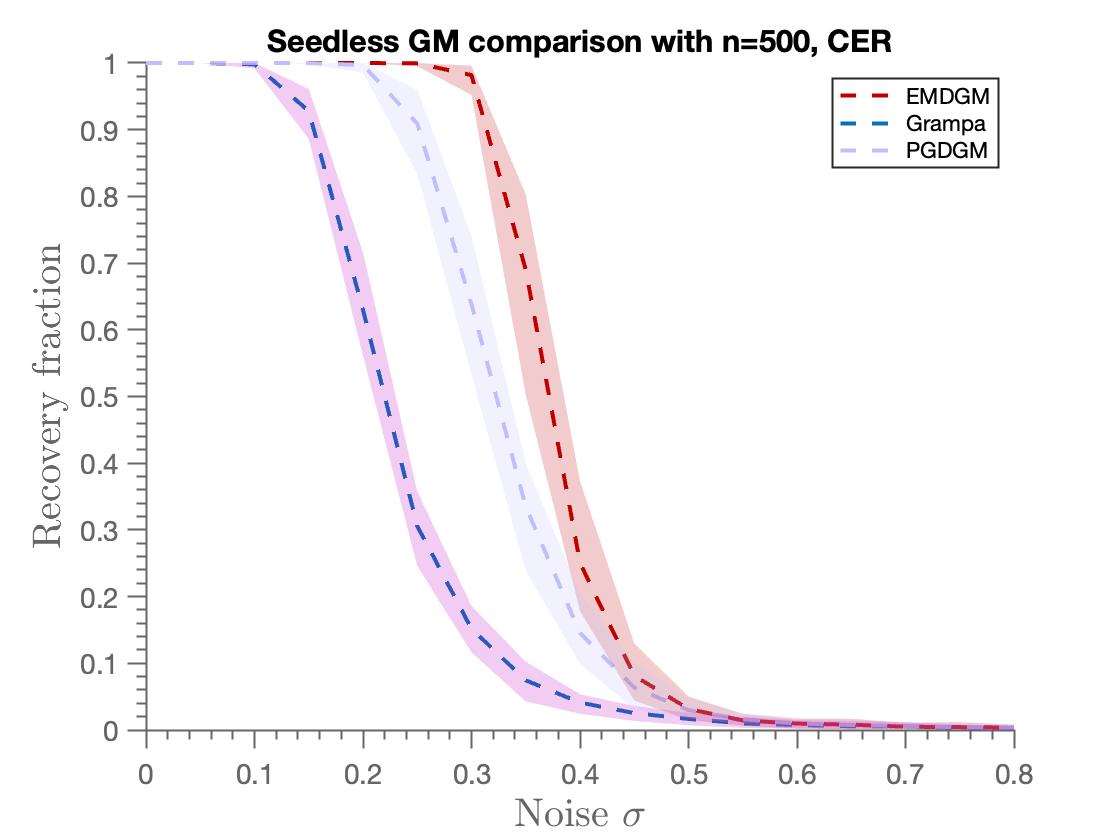}
    \caption{$N=125$ in the CER model with $p=\frac12$.}
    \label{fig:pgd_md_b}
    \end{subfigure}
    \begin{subfigure}{.5\textwidth}
        \centering
        \includegraphics[scale=0.22]{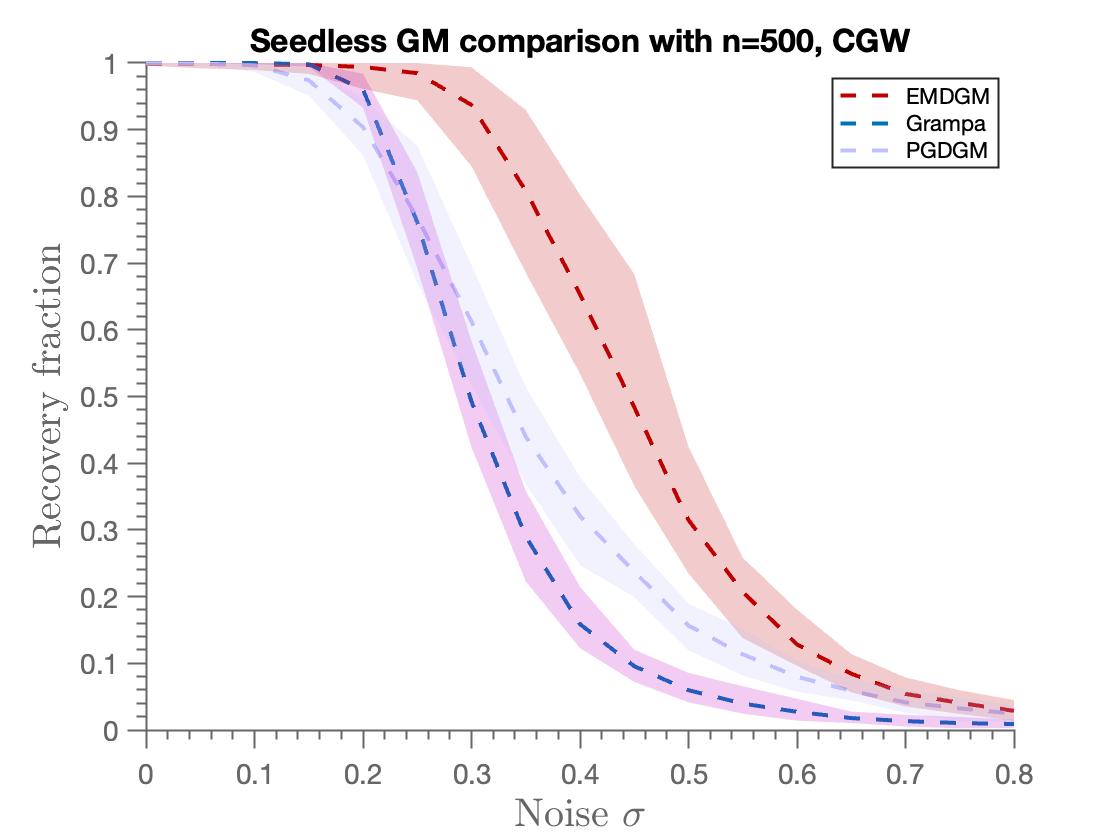}
    \caption{$N=25$ in the CGW model.}
    \label{fig:pgd_md_c}
    \end{subfigure}%
    \begin{subfigure}{.5\textwidth}
         \centering
         \includegraphics[scale=0.22]{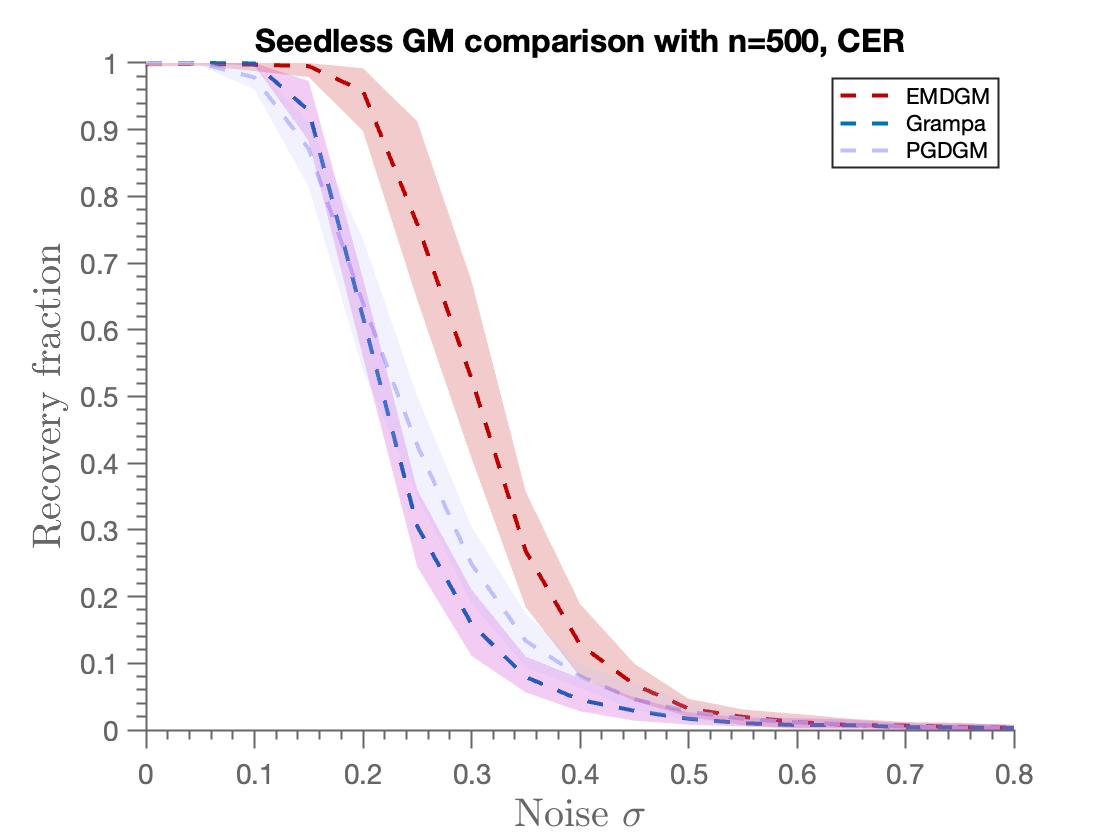}
    \caption{$N=25$ in the CER model with $p=\frac12$.}
    \label{fig:pgd_md_d}
    \end{subfigure}
    \caption{We compare $\mdgm$, $\grampa$ and PGD under the CGW and CER models for $n=500$. For $\mdgm$ we used the dynamic step-size rule \eqref{eq:dynamic_step_md} and PGD we used the heuristic rule \eqref{eq:heuristic_pgd_1} with $\theta=1$. In both cases, we fix the number of iterations to $N=25$ (Figs.\ref{fig:pgd_md_c} and \ref{fig:pgd_md_d}) and $N=125$ (Figs. \ref{fig:pgd_md_a} and \ref{fig:pgd_md_b}). The $\grampa$ regularization parameter $\gamma$ is set to $0.2$ as recommended in \cite{Grampa}. We present the average overlap (recovery fraction) for $25$ Monte Carlo runs. The shaded area indicates where $90\%$ of the data falls (excluding the top and bottom $5\%$). }
    \label{fig:comp_md_pgd}
\end{figure}
%

\paragraph{Accuracy comparison.}In Fig. \ref{fig:comp_md_pgd} we plot the average overlap between $X^*$ and the output of $\mdgm$, $\texttt{PGDGM}$ and $\grampa$. Among the step-size rules described in Section \ref{sec:exp_synthetic_data} above, the better performing ones in our experiments were the dynamic adaptive rule \eqref{eq:dynamic_step_md} for $\mdgm$ and the heuristic rule for $\texttt{PGDGM}$ with $\theta$ close to $1$ (we chose to include the default value $\theta=1$ in the plot). We see that all three algorithms perform better under the CGW model than under the CER model. Indeed, one can see in Fig. \ref{fig:pgd_md_a} that $\mdgm$ and $\texttt{PGDGM}$ can achieve perfect recovery for $\sigma\approx 0.45$, and $\grampa$ for $\sigma$ around $0.25$. Interestingly, in the case of the CER model, we see that $\mdgm$ performs better in terms of perfect recovery with a gap of at least $0.1$ in the noise level it tolerates. When we restrict the number of the iterations to $N=25$, we see that all models have an accuracy of $90\%$ or more under a noise level as high as $\sigma = 0.25$. On the other hand, as seen in Fig. \ref{fig:pgd_md_b}, for the same noise level, only the $\mdgm$ algorithm achieves $90\%$, under the setting of this experiment. Overall, $\mdgm$ outperforms both $\grampa$ and $\texttt{PGDGM}$ in this experiments. 
\paragraph{Running time comparison.} Table \ref{tab:comp_pgd_run_time} contains the  running time information for the experiments shown in Figure \ref{fig:comp_md_pgd}. We see that $\grampa$ is the fastest algorithm, running approximately five times as fast as $\mdgm$ (with $N=25$), which in turn is approximately two times faster than $\pgdgm$. Despite $\mdgm$ being slower than $\grampa$, we demonstrate here that when terminated at $N=25$, $\mdgm$ still outperforms $\grampa$, achieving an average speed five times faster than the case with $N=125$.
\begin{table}[h!]
    \centering
    \begin{tabular}{l||c|c|c|c|c}    
       &$\mdgm(N=25)$  & $\mdgm(N=125)$ &$\pgdgm(N=25)$  & $\pgdgm(N=125)$&$\grampa$ \\
       \hline
      CGW&  0.5822(0.0150) & 2.5410(0.0202) & 0.9587(0.0695)& 4.2685(0.02800)&0.1210(0.0072)\\
      CER& 0.5959(0.01621) & 2.4803(0.0157)& 0.9914(0.0723)& 4.2382(0.0215)&0.1494(0.0082)\\
    \end{tabular}
    \caption{Average (over $25$ Monte Carlos) running time in seconds (standard deviation in parentheses).}
    \label{tab:comp_pgd_run_time}
\end{table}
%
%
\subsubsection{Experimental comparison between \eqref{eq:property} and \eqref{eq:diago_dom}}\label{sec:exp_prop_vs_diagdom}
As previously noted, both $\mdgm$ and $\grampa$ adhere to the meta-strategy outlined in the Introduction: their output is the rounded version of a similarity matrix (specific for each algorithm). To prove exact recovery, a large part of the literature \cite{Grampa,YuXuLin,MaoRud,ArayaBraunTyagi} has focused on proving that the similarity matrix $\est{X}$ satisfies the diagonal dominance property \eqref{eq:diago_dom} with high probability (under different random graph model assumptions and assuming that $X^*=\id$), which we recall is
\begin{equation*}
    \est{X}_{ii}>\max_{j\neq i}\est{X}_{ij},\ \forall i\in[n].
\end{equation*}
Clearly, \eqref{eq:diago_dom} implies that $\greedy(\est{X})=\id$.
We compare \eqref{eq:diago_dom} with the property \eqref{eq:property}\rev{, which we recall requires $\est{X}_{ii}\vee \est{X}_{jj}>\est{X}_{ij}\vee \est{X}_{ji},\ \forall i\neq j\in [n]$. It is} weaker than \eqref{eq:diago_dom} yet sufficient to ensure $\greedy (\hat{X})=\id$ (as proven in Lemma \ref{lem:sufficient_prop} part $(ii)$). The objective is to see if for moderately small values of $n$, we see the diagonal dominance of the similarity matrix. The setup is as follows: we run $\mdgm$ ($N=25$) and $\grampa$ with $n=500$ using the CGW model and compute the following two metrics for the similarity matrix of each method.
\begin{itemize}
    \item \textbf{Number of non-diagonally dominant rows.} We consider the quantity 
    \begin{equation}\label{eq:metric_1}
        |\{i\in [n]: \est{X}_{ii}\leq\max_{j\neq i}\est{X}_{ij}\}|
    \end{equation}
    \item \textbf{Number of pairs not satisfying \eqref{eq:property}}. In this case, we consider 
    \begin{equation}\label{eq:metric_2}
        |\{(i,j)\in [n]^2: i\neq j\text{ and } \est{X}_{ii}\vee\est{X}_{jj}\leq \est{X}_{ij} \}|
    \end{equation}
\end{itemize}
\begin{figure}[h!]
    \centering
    \begin{subfigure}{.5\textwidth}
        \centering
        \includegraphics[scale=0.22]{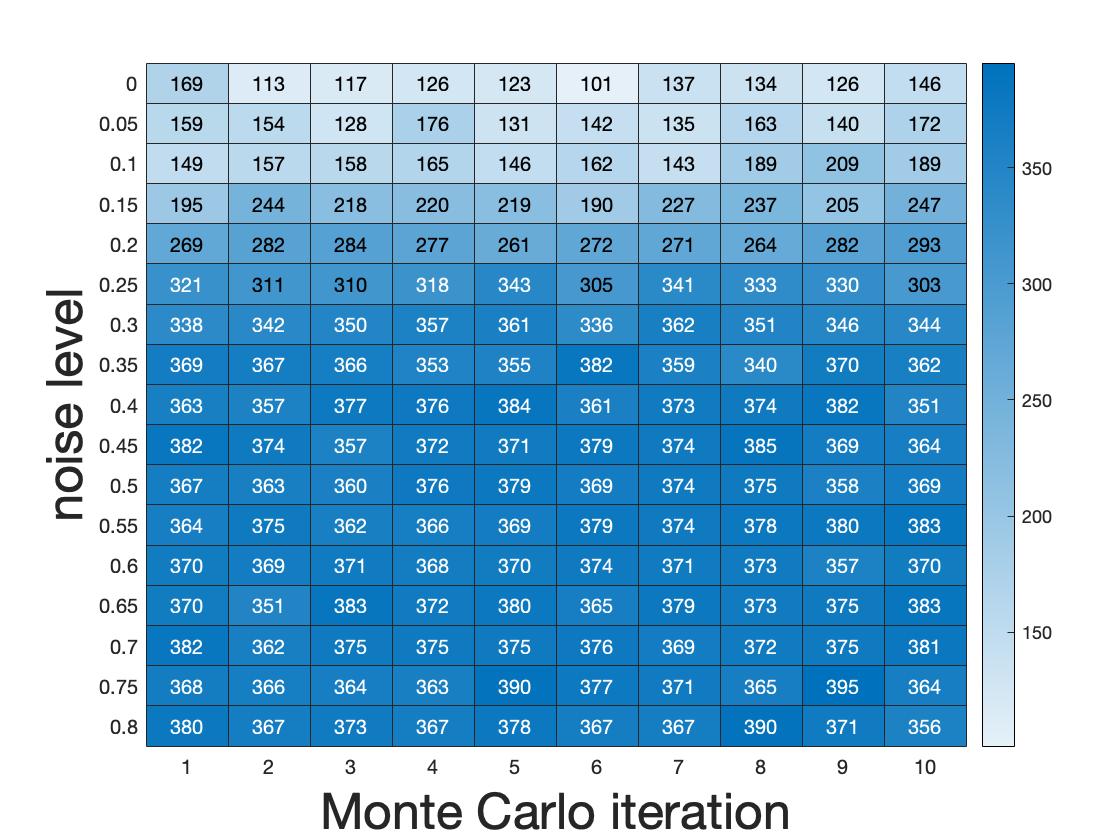}
    \caption{Metric \eqref{eq:metric_1} for $\grampa$.}
    \label{fig:comp_prop_a}
    \end{subfigure}%
    \begin{subfigure}{.5\textwidth}
         \centering
         \includegraphics[scale=0.22]{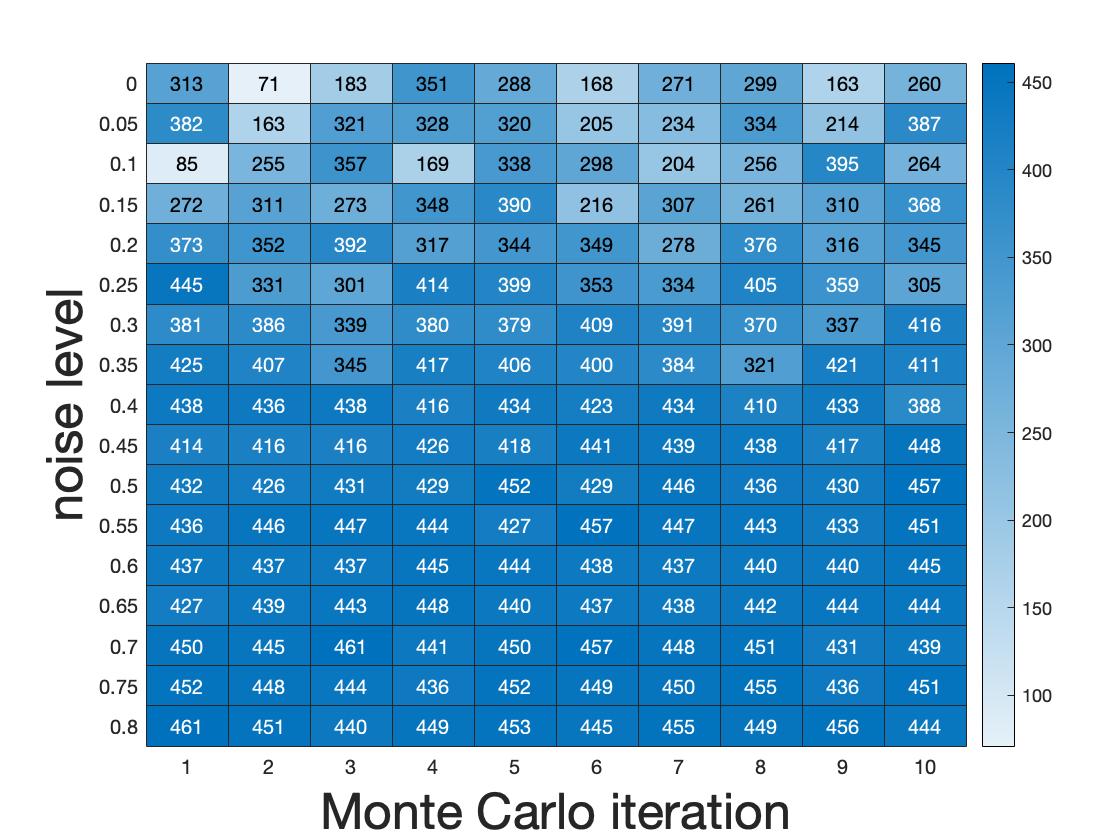}
    \caption{Metric \eqref{eq:metric_1} for $\mdgm$.}
    \label{fig:comp_prop_b}
    \end{subfigure}
    \begin{subfigure}{.5\textwidth}
        \centering
        \includegraphics[scale=0.22]{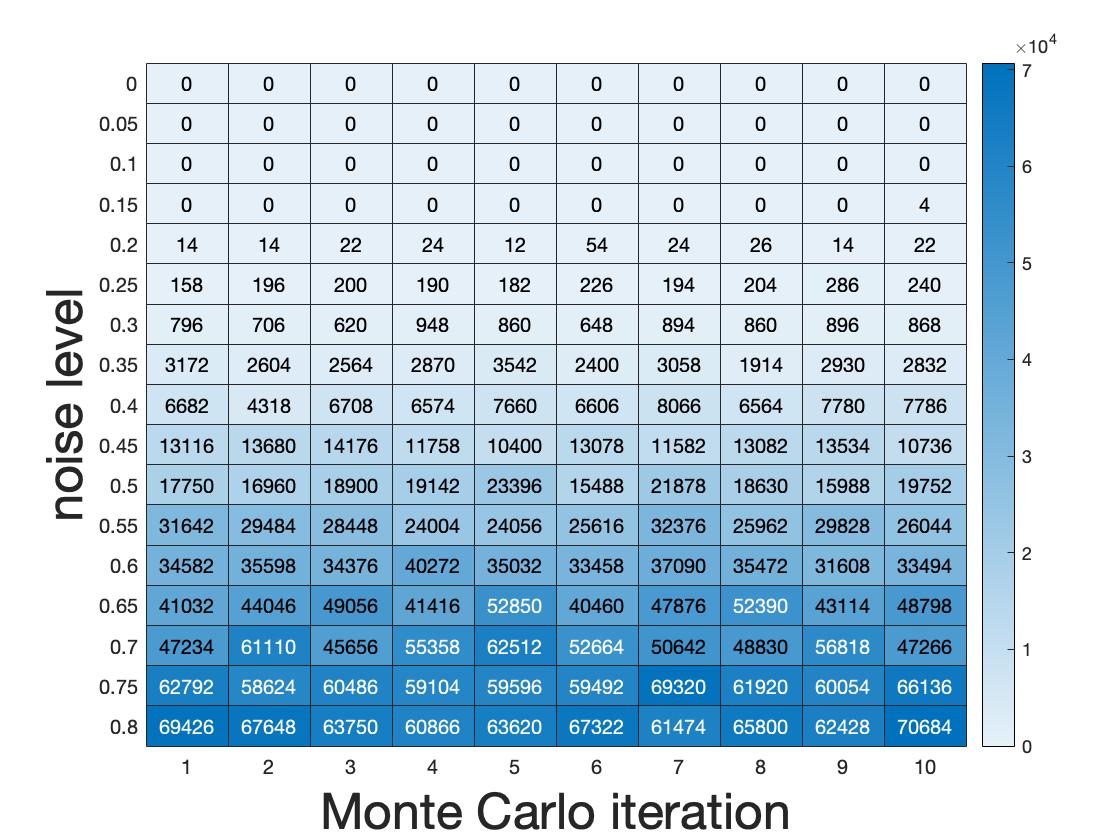}
    \caption{Metric \eqref{eq:metric_2} for $\grampa$.}
    \label{fig:comp_prop_c}
    \end{subfigure}%
    \begin{subfigure}{.5\textwidth}
         \centering
         \includegraphics[scale=0.22]{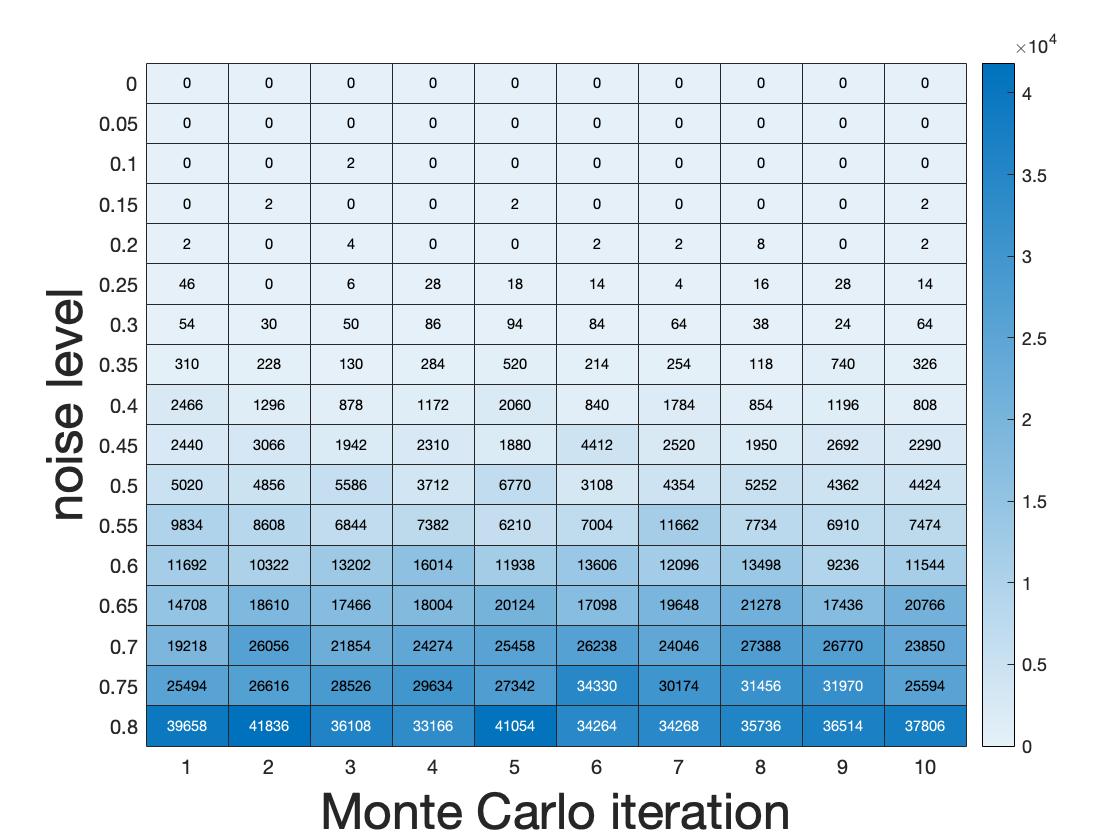}
    \caption{Metric \eqref{eq:metric_2} for $\mdgm$.}
    \label{fig:comp_prop_d}
    \end{subfigure}
    \caption{Heatmap for the metrics \eqref{eq:metric_1} and \eqref{eq:metric_2} for the similarity matrices of $\grampa$ (Figs. \ref{fig:comp_prop_a} and \ref{fig:comp_prop_c}) and $\mdgm$ with $N=125$ (Figs. \ref{fig:comp_prop_b} and \ref{fig:comp_prop_d}). The input graphs were generated using the CGW model with $n=500$ and the noise level corresponds to the parameter $\sigma$. Recall that here a value of $0$ means perfect recovery on that particular realization. }.
    \label{fig:comp_props}
\end{figure}
In Fig. \ref{fig:comp_props} we show the heatmaps for the metrics \eqref{eq:metric_1} and \eqref{eq:metric_2}, for $\grampa$ and $\mdgm$ algorithms. For both algorithms, we do not see diagonal dominance in any of the Monte Carlo samples, however in the low noise regime, both algorithms have perfect recovery (as seen in Fig. \ref{fig:pgd_md_c}). If one compares both algorithms by the metric \eqref{eq:metric_1}, the conclusion would be that $\grampa$ performs slightly better than $\mdgm$, however by the results in Fig. \ref{fig:pgd_md_c}, we know the opposite is true. On the other hand, we see in Figs. \ref{fig:comp_prop_c} and \ref{fig:comp_prop_d} that the metric \eqref{eq:metric_2} is $0$ for algorithms for values of $\sigma$ up to around $0.2$, which coincides with the results in Fig. \ref{fig:pgd_md_c}. In addition, under this metric, $\mdgm$ outperforms $\grampa$, which is also the conclusion of Section \ref{sec:exp_comparison_firstorder}. Overall, in this experiment, metric \eqref{eq:property} seems to be much better correlated with the perfect recovery cases in both methods. 
\subsection{Real data}\label{sec:exp_real_data}
This section focuses on evaluating the performance of the $\mdgm$ algorithm on real data in comparison to $\pgdgm$ and $\grampa$. Before proceeding, it is important to note a few general remarks. Firstly, since some of the datasets are relatively large (approximately ten thousand vertices per graph) and sparse, we have omitted the standardization preprocessing discussed in Section \ref{sec:exp_synthetic_data} to more effectively leverage their sparsity. Secondly, unless specified otherwise, we employ dynamic step sizes derived from \eqref{eq:dynamic_step_md} for $\mdgm$ and \eqref{eq:heuristic_pgd_1} for $\pgdgm$. This choice is based on their performance with synthetic data. The $\grampa$ regularizer $\eta$ is fixed to $0.2$. 
\subsubsection{Computer vision dataset}\label{sec:exp_computer_vision}
We utilize the SHREC'16 dataset \cite{shrec}, comprising 25 shapes (in both high and low resolution) representing a child in various poses. The goal is to match corresponding body parts when presented with two different shapes. The initial step in our pipeline involves converting the triangulated representation of an image (as stored in the database) into an adjacency matrix, a common procedure in computer vision (refer to \cite[Section 1]{Peyre}). We opt for the use of low-resolution versions of each image, as they are already represented by adjacency matrices with an average size of approximately $10,000$ vertices. It is worth noting that, on average, the number of edges in these matrices is only $0.05\%$ of the total possible edges, indicating a high degree of sparsity. To leverage this sparsity, we will omit the input standardization described in Section \ref{sec:exp_synthetic_data}. Another important aspect of the SHREC'16 dataset is that all images have varying sizes, resulting in inputs of the form $A\in\mathbb{R}^{n_A\times n_A}$ and $B\in\mathbb{R}^{n_B\times n_B}$ with $n_A\neq n_B$. It is worth noting that both $\mdgm$ and $\pgdgm$ naturally extend to accommodate this variation, and the same can be said for $\grampa$ (see \cite[Section 4]{Grampa}). We will apply and compare the $\mdgm$, $\pgdgm$, and $\grampa$ algorithms using the Princeton benchmark protocol \cite{princeton_protocol}. This benchmark protocol essentially involves a CDF comparison, defined by the following steps.
\begin{itemize}
    \item \textbf{Normalized geodesic errors. }Given inputs $A$ and $B$, and the output $\hat{\pi}$ of a graph matching algorithm, we calculate the normalized geodesic error as follows: for each node $i$ in $A$, we compute $d_{B}(\hat{\pi}(i),\pi^*(i))$, where $\pi^*$ represents the ground truth matching, and $d_{B}$ denotes the geodesic distance on $B$ (computed as the weighted shortest path distance using the triangulation representation of the image \cite{Peyre}). The normalized error is defined as \[\varepsilon(i):=d_{B}(\hat{\pi}(i),\pi^*(i))/\sqrt{\text{Area}(B)},\] where $\text{Area}(B)$ is the surface area of $B$ (computed using the triangulation representation).
    \item \textbf{CDF computation. }The error cumulative distribution function (CDF) is defined for $t\geq 0$ as
\begin{equation}\label{eq:error_cdf}
\text{CDF}(t) = \sum^{n_A}_{i=1}\mathbbm{1}_{\varepsilon(i)\leq t},
\end{equation}
where $n_A$ is the number of nodes in $A$. Notice that $\operatorname{CDF}(0)$ corresponds to the recovery fraction, or overlap as defined in Section \ref{sec:exp_synthetic_data}. 
\end{itemize} 
\begin{figure}
    \centering
    \includegraphics[scale=0.24]{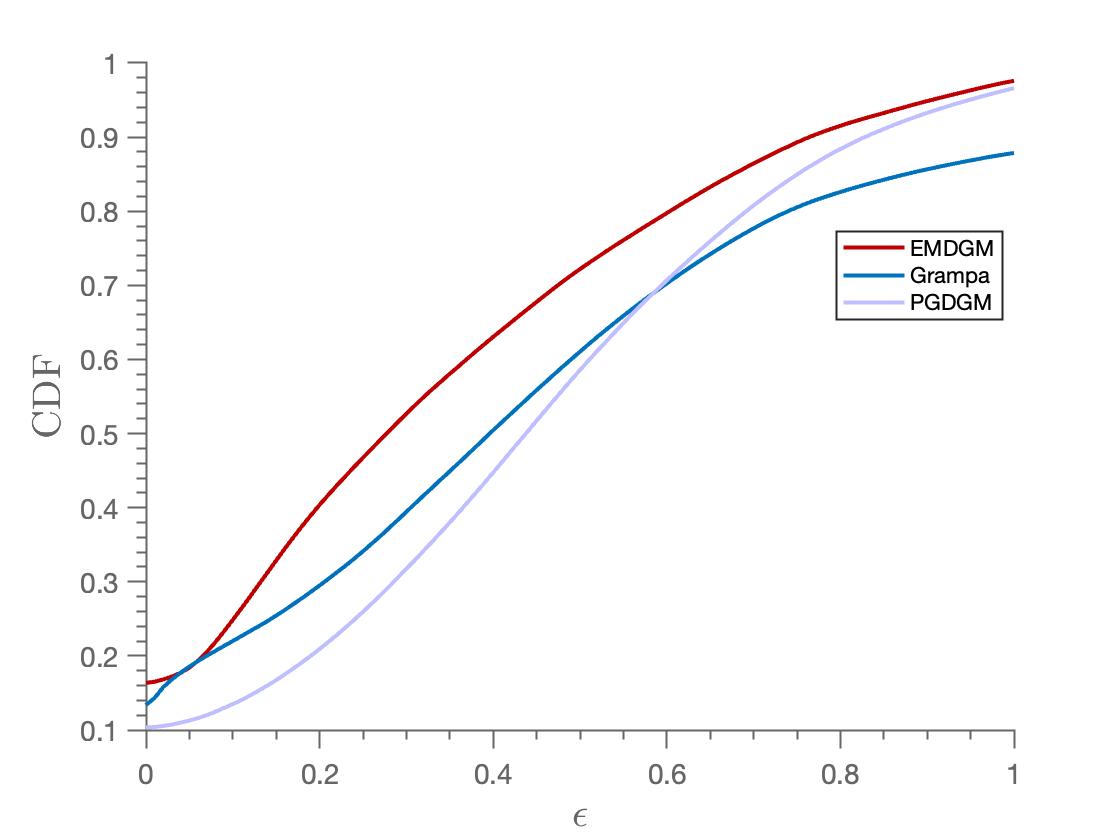}
    \caption{Performance comparison of $\mdgm$, $\pgdgm$, and $\grampa$ on the SHREC'16 dataset, based on the error cumulative distribution function (CDF) defined in \eqref{eq:error_cdf}. We conducted 100 iterations each for $\mdgm$ and $\pgdgm$.}
    \label{fig:comp_shrec16}
\end{figure}

In Fig.\ref{fig:comp_shrec16} we show the average performance of the three algorithms on the SHREC'16 dataset. We see that the performance of $\mdgm$ is the best of the three, while for a small value of $\epsilon$ (around $0.05$), $\grampa$ has a similar performance. It is interesting that for the other values of $\epsilon$ mirror descent has a big gap in performance compared to $\grampa$, because in the graph matching problem we are not trying to explicitly optimize the errors beyond $\epsilon=0$ (which by the definition \eqref{eq:error_cdf} is related to exact recovery). In this case, $\mdgm$ appears to introduce an additional level of regularity beyond what is specified in the optimization problem. We attribute the suboptimal performance of $\pgdgm$ to our difficulty in properly tuning the step-sizes.
\subsection{Autonomous systems}\label{sec:exp_aut_systems}
We use data from the Autonomous Systems dataset, which belongs to the University of Oregon Route Views Project \cite{AS_website}, and is also available at the Stanford Network Analysis Project repository \cite{AS_stanford,AS_lesko}. This data has also been used to evaluate the performance of $\grampa$ and other graph matching algorithms in \cite[Section 4.5]{Grampa}. The dataset consists of a dynamically evolving network of autonomous systems (a collection of IP networks under a common routing policy) that exchange traffic through the Internet. There are a total of nine graphs, with each one representing weekly observations of the autonomous system from March 31, 2001, to May 26, 2001. The node count varies between $10,670$ and $11,174$, while the edge count fluctuates between $22,002$ and $23,409$ over time. Since the number of nodes and edges exhibits these fluctuations over time (see \cite{AS_stanford} for more summarizing statistics), we have implemented a subsampling strategy adopted in \cite[Section 4.5]{Grampa}. We focus on a subset of $1,000$ vertices with the highest degrees, found in all nine networks. The nine resulting graphs can be seen as perturbations of each other. Fixing the first graph (from March 31), as input $A$, and comparing it to the subsequent ones, we observe that the correlation approximately decreases over time (although is not fully monotone). This provides a comparison framework similar to the synthetic data experiments, where the correlation was controlled by the parameter $\sigma$. 
\begin{figure}[h!]
    \centering
        \begin{subfigure}{.5\textwidth}
        \centering
        \includegraphics[scale=0.22]{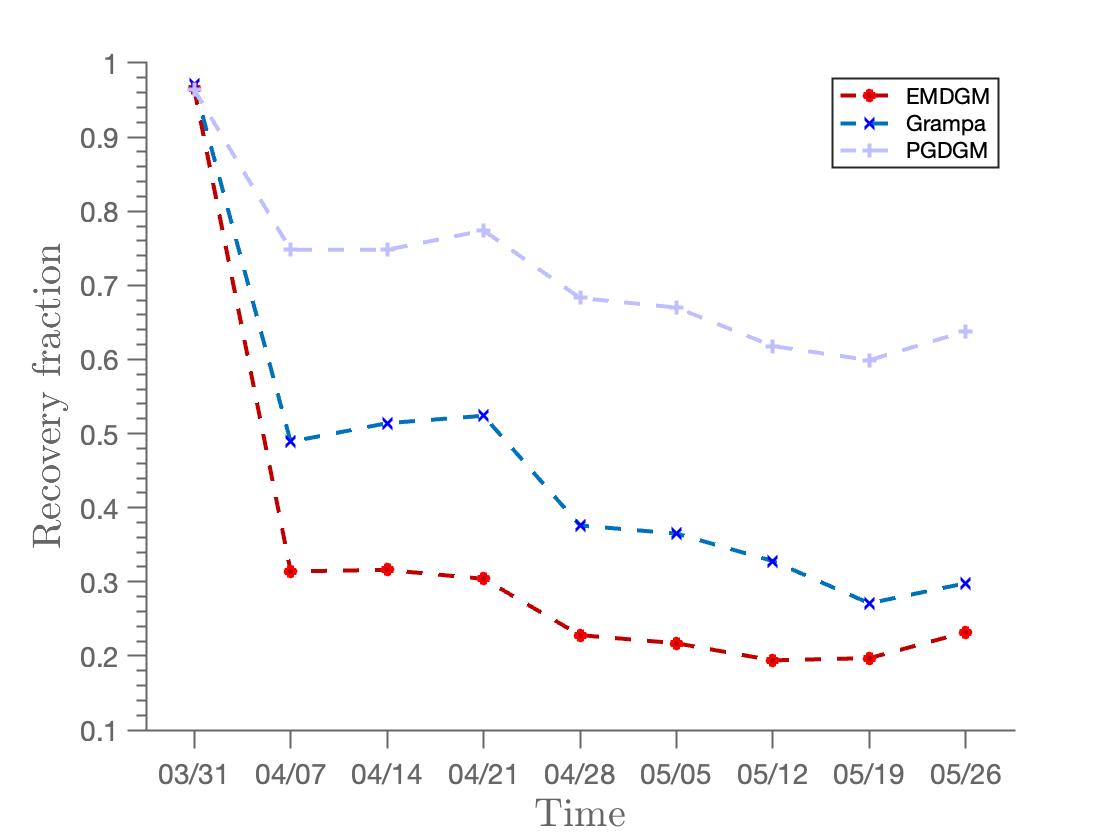}
    \caption{Recovery fraction in a high-degree subsampled graph.}
    \label{fig:comp_auto_a}
    \end{subfigure}%
    \begin{subfigure}{.5\textwidth}
         \centering
         \includegraphics[scale=0.22]{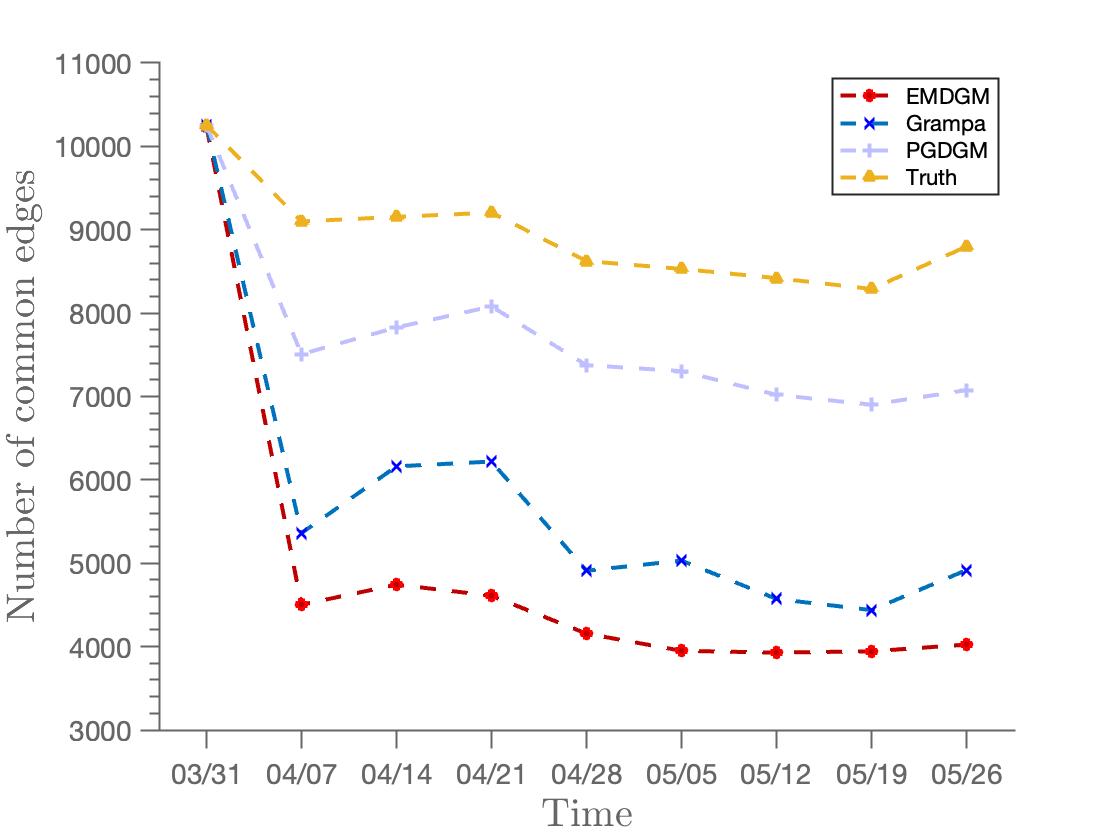}
    \caption{Common edges in a high-degree subsampled graph.}
    \label{fig:comp_auto_b}
    \end{subfigure}
    \begin{subfigure}{.5\textwidth}
        \centering
        \includegraphics[scale=0.22]{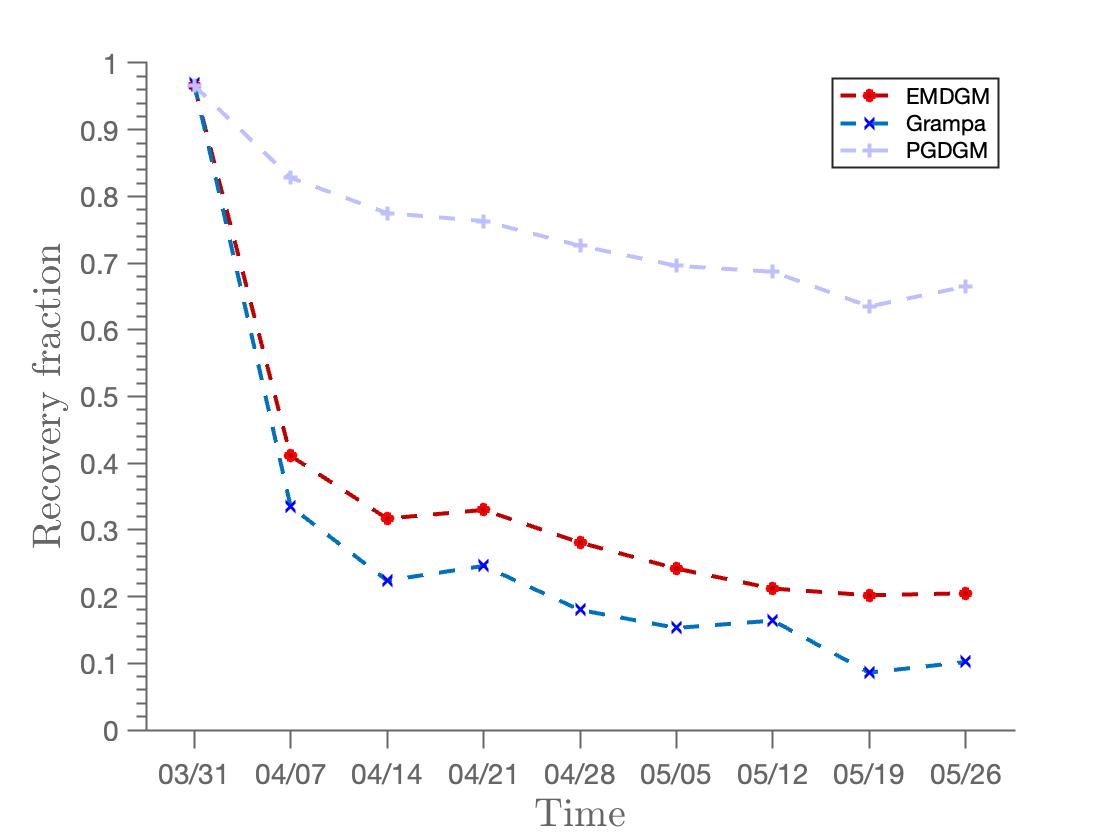}
    \caption{Recovery fraction in a high-degree standardized graph.}
    \label{fig:comp_auto_c}
    \end{subfigure}%
    \begin{subfigure}{.5\textwidth}
         \centering
         \includegraphics[scale=0.22]{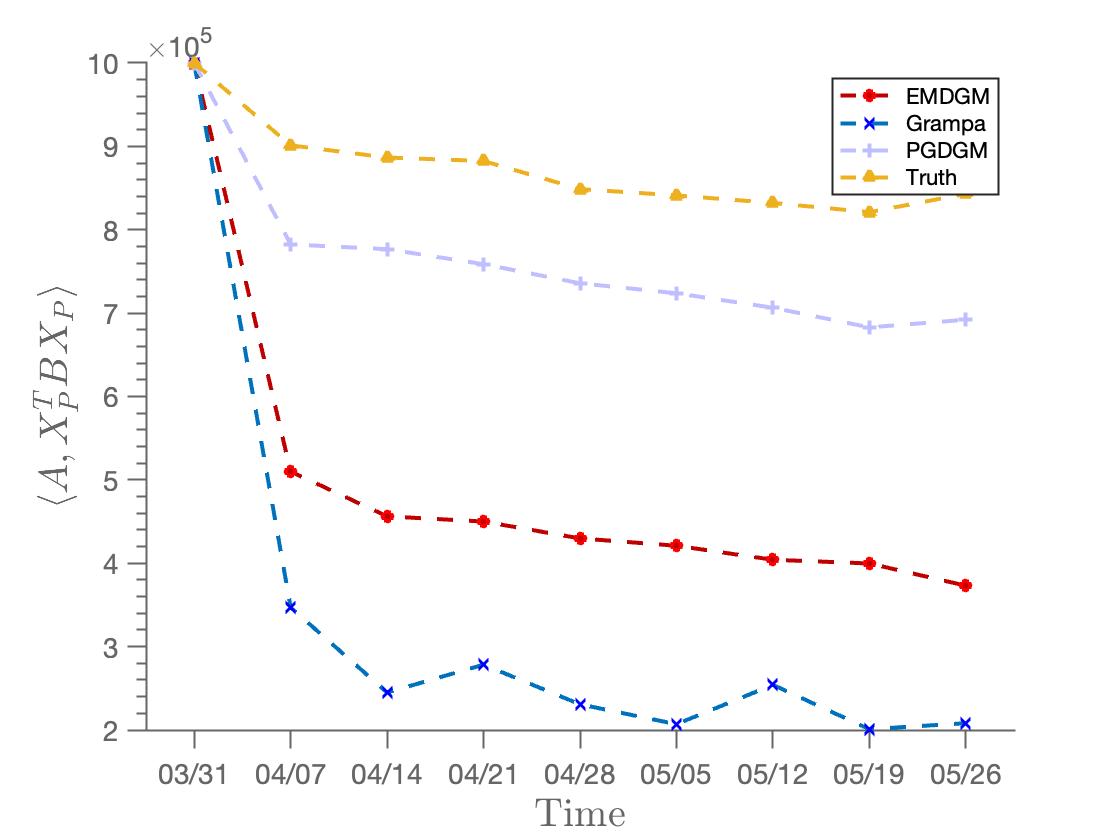}
    \caption{Rescaled objective in a high-degree standardized graph.}
    \label{fig:comp_auto_d}
    \end{subfigure}
    \caption{Comparison of $\mdgm$, $\grampa$, and $\pgdgm$ using a high-degree subsampled graph with $1,000$ vertices for each time instant. The input graph $A$ remains fixed as the graph from March 31, while the input graph $B$ varies (with random permutations of its rows and columns), as indicated on the x-axis.  We used the dynamic step rule for $\mdgm$, derived from \eqref{eq:dynamic_step_md} and the heuristic rule \eqref{eq:heuristic_pgd_1} for $\pgdgm$, with $125$ iterations in both cases.}
    \label{fig:comp_auto}
\end{figure}

In Figure \ref{fig:comp_auto} we compare the output of the three methods $\mdgm$, $\grampa$, $\pgdgm$ under two different metrics. The first is the recovery fraction, plotted in Figure \ref{fig:comp_auto_a} and the other is the rescaled objective $\langle A,{\est{X}_\calP}^TB\est{X}_\calP\rangle_F$ (here $\est{X}_\calP$ represent the output of a graph matching algorithm), plotted in Figure \ref{fig:comp_auto_b}. As noted in \cite{Grampa}, the latter metric is more instructive for this dataset due to the presence of numerous degree-one vertices connected to high-degree vertices, resulting in symmetries that render the ground truth unidentifiable. In this experiment, we apply all three graph matching algorithms to nine instances created by fixing input $A$ as the first graph. Input $B$ is obtained by random node relabeling for each of the nine graphs. In Figures \ref{fig:comp_auto_c} and \ref{fig:comp_auto_d} we plot the same metrics, but  with standardized inputs. For that, we subtract to each matrix their average value and divide them by their standard deviation. Notably, the performance of $\mdgm$ and $\pgdgm$ improves, while that of $\grampa$ declines. This standardization alters the objective function and appears to introduce a `regularization' effect for the methods on the simplex. Further investigation into this phenomenon is warranted, but we leave it for future research. 

While we do not have a full explanation for why $\pgdgm$ outperforms $\mdgm$ in this case, we can provide some partial insights. In certain scenarios, the ``regularity'' of the objective function may be more influential for achieving better results, whereas in other cases, the performance disparity may be primarily attributed to the selection of an appropriate step-size strategy. We believe that the latter plays a greater role in this case, based on the following experiment. We fix the step-size rule, using \eqref{eq:fixed_iter_md} for MD and \eqref{eq:fixed_iter_pgd}, and define the efficiency ratio of the function $E$ (as seen in \cite[Example 9.17]{First_order}) as:

\begin{equation*}
\rho_E = \frac{\sqrt{\log n}, L_{E,\infty}}{L_{E,2}},
\end{equation*}
where it is clear that $\frac{\sqrt{\log{n}}}{\sqrt{n}} \leq \rho_E \leq \sqrt{\log{n}}$. When $\rho_E$ is closer to the lower bound, MD tends to outperform PGD, while the reverse is true when $\rho_E$ is closer to the upper bound. We compute an empirical version of the unobserved quantities $L_{E,\infty}$ and $L_{E,2}$ (which can be regarded as regularity measures for $E$), by taking $10,000$ uniformly distributed samples on $\simpn2$ and computing the maximum of the $\ell_\infty$ and $\ell_2$ norms of the gradient evaluated on those points. The value of $\rho_E$ obtained varies from $0.46$ to $0.53$, while $\frac{\sqrt{\log n}}{\sqrt{n}}=0.0831$ and $\sqrt{\log n}=2.62$. Using this metric, there is not a clear advantage between the two methods. However, when applying rule \eqref{eq:fixed_iter_pgd}, we observe that $\pgdgm$ significantly underperforms $\mdgm$ (approximately ten times worse). It appears that the heuristic rule \eqref{eq:heuristic_pgd_1} plays a crucial role in accelerating the convergence of $\pgdgm$. Unfortunately, we currently lack a similar rule for $\mdgm$, but this is an avenue for future exploration. 


%
\subsection{Facebook networks}\label{sec:exp_fb}
We use Facebook's friendship network data from \cite{facebook_data}, consisting of $11,621$ individuals affiliated with Stanford University. In contrast to the previous examples, where multiple networks were available, here, we have a single large network. As a result, we need to employ a subsampling strategy to generate correlated graphs for evaluating the graph matching methods. In the context of seeded graph matching, \cite{YuXuLin} employed a specific sampling strategy to evaluate methods based on seed quality. However, since our objective is different, our sampling strategy also differs from theirs. The complete dataset contains $136,660$ (unweighted) edges, with a degree distribution closely resembling a power-law pattern (density proportional to $d^{-1}$), and an average degree of approximately $100$. We obtain two correlated graphs $A,B$ as follows. First, we create graph $H$ by selecting $1,000$ vertices uniformly at random, forming the induced subgraph. From the parent graph $H$, we retain each edge with a probability of $s$ to create graph $A$. We independently repeat this process to generate graph $B$ (as usual, we shuffle the labels of $B$ uniformly before applying the graph matching algorithms). Intuitively, a higher value of $s$ corresponds to a stronger correlation between the graphs, with $s=1$ implying isomorphism. Notice that our sampling scheme closely resembles the ``parent graph'' definition in \cite{PedGloss} for the CER model.
\begin{figure}[h!]
    \centering\includegraphics[scale=0.24]{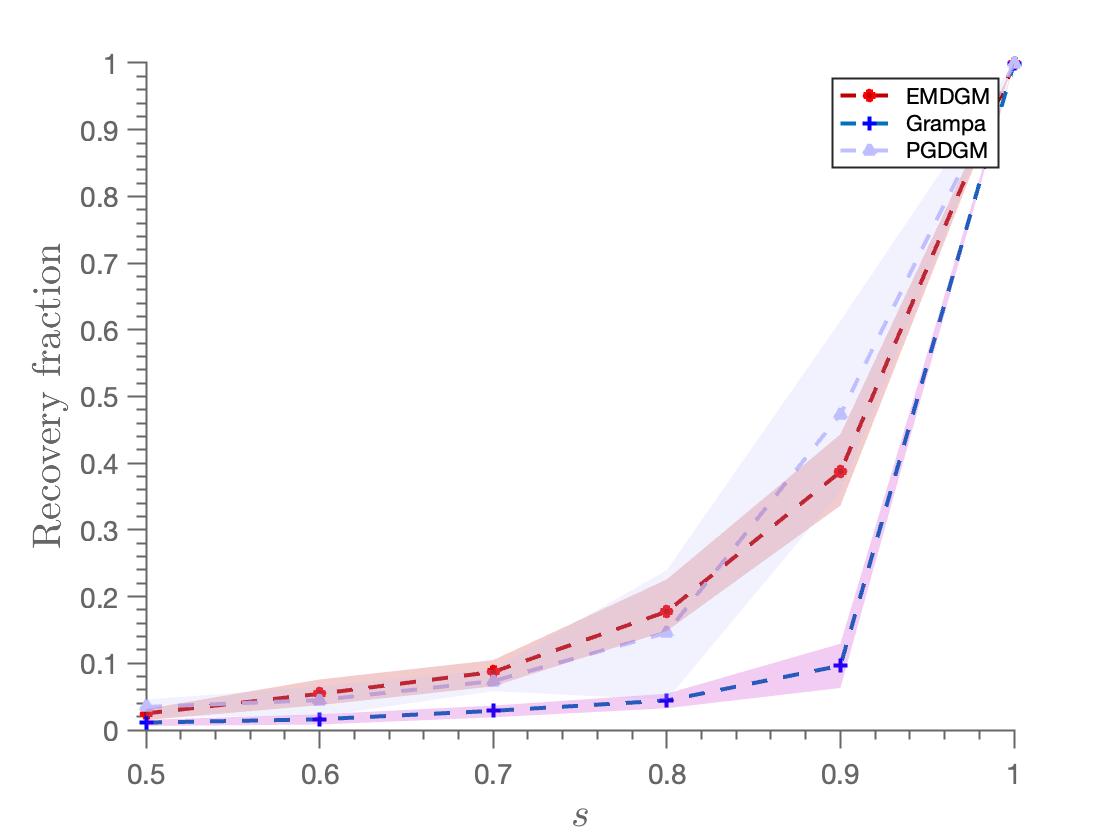}
    \caption{Comparison of graph matching algorithms on subsamples of $1,000$ edges of a Facebook friendship network. We plot the average over $15$ Monte Carlo (over the uniform subsamples). We used the dynamic step rule for $\mdgm$, derived from \eqref{eq:dynamic_step_md} and the heuristic rule \eqref{eq:heuristic_pgd_1} for $\pgdgm$, with $125$ iterations in both cases.}
    \label{fig:comp_fb}
\end{figure}

In Figure \ref{fig:comp_fb} we plot the recovery fraction for $\mdgm$, $\pgdgm$ and $\grampa$. As expected, all three methods achieve exact recovery in the isomorphic case $s=1$. However, for smaller values of $s$, there is a significant gap in performance between $\mdgm$ and $\pgdgm$ (which are comparable) and $\grampa$. The more pronounced drop in performance, compared to the synthetic data experiments presented in Section \ref{sec:exp_synthetic_data}, is also expected given the power-law behavior exhibited by this network. An interesting observation can be made from the confidence intervals in Figure \ref{fig:comp_fb}: $\mdgm$ shows less variance across the Monte Carlo subsamples. To quantify this, we computed the average standard deviation over the $6$ instances of $s$. For $\mdgm$, this value is $0.0126$, while for $\pgdgm$, it is significantly higher at $0.0604$.

\section{Concluding remarks}\label{sec:conclusion}
In this work, we present a novel graph matching algorithm based on a convex relaxation of the permutation set to the unit simplex. We introduce an iterative mirror descent scheme that offers closed-form expressions for the iterates, resulting in enhanced efficiency compared to competing projecting gradient methods. Theoretical analysis demonstrates that, when coupled with a standard greedy rounding procedure, our method can exactly recover the ground truth matching in the noiseless case of the correlated Gaussian Wigner model. Experimental results support this claim, showing that our proposed algorithm can achieve exact recovery even in the presence of significant noise. In terms of statistical performance (\rev{as captured by the recovery fraction}), it consistently outperforms other convex relaxation-based methods with similar time complexity, \rev{such as \texttt{Grampa}, \texttt{Umeyama} and several other considered in \cite[Section 5]{Grampa} (c.f. Remark \ref{rem:exp_other_algos}).} 

The following are potential avenues for future research. 
\begin{itemize}
    \item An intriguing theoretical question is to extend the guarantees to the case where $\sigma\neq 0$ in the CGW model. This presents two main challenges. First, understanding the eigenvector structure of the iterates in \eqref{eq:md_dynamic_update} is nontrivial, given its Hadamard product definition and the complexity of the gradient expression. Second, using Lemma \ref{lem:sufficient_prop} becomes less straightforward when not all pairs satisfy the condition \eqref{eq:property}. To address this, one possible approach is to demonstrate that the number of pairs not satisfying \eqref{eq:property}, as given in \eqref{eq:metric_2}, converges to zero. However, this seems to present significant technical challenges.
    \item Extending recovery guarantees to other models is also an important theoretical objective. A common approach in the literature of iterative methods involves using uniform concentration bounds, which is essentially a worst-case scenario approach. To follow this strategy, it will be crucial to gain a deeper understanding of the structure of the iterates in \eqref{eq:md_dynamic_update} beyond $\Xone$. In particular, we need to improve our ability to ``localize'' $\Xk$. Otherwise, there is a risk that, in the worst-case scenarios, the gradient may point in the wrong direction. It is not hard to provide examples of such scenarios.
    \item From an experimental perspective, it is intriguing to explore acceleration in the mirror descent dynamic. Although accelerated MD schemes have been considered in the literature, as discussed in \cite{Krichene}, they involve Bregman projections that are fast in theory but lack a closed form, making their implementation more involved. It would be interesting to assess the practical performance of these schemes. 
    \item Another experimental question involves finding a rate selection method for EMD that enhances its statistical performance, similar to the heuristic \eqref{eq:heuristic_pgd_1} in the case of PGD. In general, any other techniques used to improve the performance of first-order methods can also be tested here, including adaptive restart, early stopping, and more.
\end{itemize}

\clearpage
\bibliographystyle{plain}
\bibliography{biblio}

\newpage
\appendix
%
\section{The population dynamics}
We introduce and analyze the population (or average) version of the mirror descent dynamics. For that, notice that all randomness in the dynamics defined by \eqref{eq:md_dynamic_update} is given by the input matrices $A,B$. In addition, the dependence on these matrices in the iterative process is captured by the term $\nabla E(X)=A^2X+XB^2-2AXB$ (as it is clear from line $3$ of Algorithm \ref{alg:mirror_simplex}). We define the \emph{population dynamics}, by replacing the (random) loss function $E(X)$ by the (deterministic) \emph{population loss}, defined as $\expec_{A,B} E(X)$. This is equivalent to replace the term $\nabla E(X)$ in \eqref{eq:md_dynamic_update} by its expectation $\popgradient(X):=\expec_{A,B} \nabla E(X)$, which is proven as part of Lemma \ref{lem:population_gradmatrix}. More formally, this dynamic, which is deterministic, is defined given a sequence  $(\gamma_k)^\infty_{k=1}$ of positive rates, by the recursion
\begin{align}\label{eq:pop_dyn_init}
\Xzero&=J/n^2,\\ \label{eq:pop_dyn_update}
\Xkone&=n\cdot\frac{\Xk\odot \hadexp\left(-\gamma_k\popgradient(\Xk)\right)}{\|\Xk\odot \hadexp\left(-\gamma_k\popgradient(\Xk)\right)\|_{1,1}}.
\end{align}
For each $k$, we define the normalization factor by
\[\overline{N}_k:=\frac1n \|\Xk\odot \hadexp\left(-\gamma_k\popgradient(\Xk)\right)\|_{1,1},\]
which is a positive real number.
%
\subsection{One iteration ($N=1$)}
The following result shows that Algorithm \ref{alg:mirror_simplex} succeeds in finding the ground truth $X^*$ in one step (i.e., when $N=1$). We will assume without loss of generality that $X^*=\id$, and consequently in the definition of the population gradient, we have$\popgradient(X)=\expec_{A,B} \nabla E(X)$, where the joint law of $A,B$ is $\calW(n,\sigma,X^*=\id)$. The following result shows that in the case of the population dynamic, one step is sufficient for recovering the ground truth if we applied Algorithm \ref{alg:gmwm} directly after the first iteration.
\begin{proposition}\label{prop:pop_dyn_onestep}
Let $\Xone\in \matR^{n\times n}$ be the first iterate in the dynamic defined by \eqref{eq:pop_dyn_init} and \eqref{eq:pop_dyn_update}, defined by $\gamma_0\in \matR_+$. Then $\greedy(\Xone)=\id$, for any $\gamma_0>0$. 
\end{proposition}
\begin{proof}
From \eqref{eq:pop_dyn_init} and \eqref{eq:pop_dyn_update}, we have 
\[\Xone=\overline{N}^{-1}_0\Xzero\odot \hadexp\left(-\gamma_0\popgradient(\Xzero)\right),\]
where $\overline{N}_0>0$ is the normalization term. Using the definition for the population gradient, Lemma \ref{lem:population_gradmatrix} and by a simple calculation, we get 
\begin{align*}
\Xone &= \oN^{-1}_0J\odot \hadexp\left(\frac{-\gamma_0}{n^2}\big((2+\frac{n+1}{n}\sigma^2)J-2\id\big)\right)\\
&= \oN^{-1}_0 J\odot \hadexp\left(\frac{-\gamma_0}{n^2}(2+\frac{n+1}{n}\sigma^2)J\right)\odot \hadexp\left(\frac{2\gamma_0}{n^2}\id\right)\\
&=\oN'_0 \hadexp\left(\frac{2\gamma_0}{n^2}\id\right),
\end{align*}
where we define $\oN_0':=\frac{\oN^{-1}_0}ne^{-\frac{\gamma_0}n(2+\frac{n+1}n\sigma^2)}$. 
Given that $e^{\frac{2\gamma_0}{n^2}}>1$ for $\gamma_0>0$, we deduce that $\Xone$ satisfies \eqref{eq:diago_dom}, which implies that $\greedy(\Xone)=\id$.
\end{proof}
%
\subsection{Multiple iterations ($N>1$)}
In practice, we do not have a priori information on how many iterations we need to perform before the rounding step, but we expect Algorithm \ref{alg:mirror_simplex} to be relatively robust with respect to the choice of the number of iterations $N$ (at least in some range), for some choice of the step size sequence $(\gamma_k)^N_{k=1}$. The next result proves that this is the case for the population dynamic, in the sense that all its iterates can be rounded to the ground truth, provided that a certain condition on the rate sequence $(\gamma_k)_{k\geq 0}$ holds. 
%
\begin{theorem}
\label{thm:dyn_multistep}
Consider $\Xzero,\Xone,\cdots,X^{(N)}\in\matR^{n\times n}$ the iterates of the dynamic defined by \eqref{eq:pop_dyn_init} and \eqref{eq:pop_dyn_update}, defined by the rate sequence $(\gamma_k)^{N-1}_{k=0}\in\matR_+$. If $\sigma \leq 1$ and it holds
\begin{equation}\label{eq:condition_rates}
\sum^{N-1}_{k=0}\gamma_k<\frac{n-1}{4}\log{2},
\end{equation}
then $\greedy(\Xk)=\id$, for all $1\leq k\leq N$.
\end{theorem}
Before proving this theorem, we will prove two lemmas that characterize the structure of iterates in the population dynamic.
\begin{lemma}\label{lem:Xk_pop_dyn}
    Let $(\Xk)^\infty_{k=0}$ be a trajectory of the population dynamic given by some step size sequence $(\gamma_k)^\infty_{k=0}\in \matR_+$. We have that $\Xk$ is symmetric for all $k$ and that \begin{equation}\label{eq:Xk_pop_dyn}
        \Xk_{ij}=\begin{cases}
        \Xk_{11}\text{ if }i=j\\
        \Xk_{12}\text{ if }i\neq j
        \end{cases}
    \end{equation}
\end{lemma}
\begin{proof}
    The proof is by induction. Clearly $\Xzero$ is symmetric. Assume that $\Xk$ is symmetric. Notice that 
    \begin{equation}\label{eq:Xkone_Xk}
        \Xkone=\oN^{-1}_k\Xk\odot \hadexp\left(-\gamma_k\popgradient(\Xk)\right),
    \end{equation}
    for a constant $\oN^{-1}_k>0$. The proof follows by the inductive hypothesis, by the fact that $\popgradient(X)$ is symmetric if $X$ is symmetric (see \eqref{eq:lemma_gradmat3}) and the fact that the Hadamard product of two symmetric matrices is symmetric. We also prove \eqref{eq:Xk_pop_dyn} by induction. The base case $k=0$ is evident. We assume that $\Xk$ satisfies the property \eqref{eq:Xk_pop_dyn}. Using \eqref{eq:Xkone_Xk}, the fact that $\popgradient(X)$ satisfies \eqref{eq:Xk_pop_dyn} if $X$ does, and the fact that \eqref{eq:Xk_pop_dyn} is closed under Hadamard products, the conclusion follows.
\end{proof}
The previous lemma asserts that each entry of $\Xk$ can take only two values, depending on its location inside the matrix (diagonal or off-diagonal). In this case, it is easy to see that the greedy rounding of $\Xk$, for any $k$, will be equal to the identity (the ground truth) if and only if $\Xk_{11}>\Xk_{12}$, which is in fact equivalent to say that $\Xk$ is diagonally dominant, i.e. for all $i\in [n]$ we have $\Xk_{ii}>\Xk_{ij}$, for all $j\neq i$.
Since the iterates of the population dynamics $\Xk$, up to the step $K$ (for a given $K\in\matN$), are completely determined by the sequence of learning rates $(\gamma_{k})^{K}_{k=0}$, we can characterize them as a function of them. To prove the diagonal dominance, it will be convenient to define the ratio between the off-diagonal and diagonal entries, for $k\geq 1$, as follows,
\[r_k(\gamma_0,\cdots,\gamma_{k-1}):=\frac{\Xk_{12}}{\Xk_{11}}.\]
Note that, in the previous definition, $r_k$ is a function from $\matR^k$ to $\matR$. The following lemma provides a recursive characterization for these ratios. 
%
\begin{lemma}\label{lem:recursion}
    For any given sequence $(\gamma_k)^\infty_{k=0}$, the ratio between off-diagonal and diagonal entries $r_k$ defined above, satisfy the following 
    \begin{align*}
        r_k(\gamma_0,\cdots,\gamma_{k-1})&=r_{k-1}(\gamma_0,\cdots,\gamma_{k-2})e^{-\gamma_{k-1}\frac{a_\sigma r_{k-1}(\gamma_0,\cdots,\gamma_{k-2})-(a_\sigma-2)}{n(n-1)r_{k-1}(\gamma_0,\cdots,\gamma_{k-2})+n}},\text{ for }k\geq 2\\
        r_1(\gamma_0)&=e^{-2\gamma_0/n^2},
    \end{align*}
    where $a_\sigma:=2+\frac{n+1}n\sigma^2$.
\end{lemma}
\begin{proof}
    To save notation, we will write $r_k$ instead of $r_k(\gamma_0,\cdots,\gamma_{k-1})$. Using Lemma \ref{lem:Xk_pop_dyn} and \eqref{eq:lemma_gradmat3} we get that 
    \[\popgradient(\Xk)_{12}-\popgradient(\Xk)_{11}=a_\sigma \Xk_{12}-(a_\sigma-2)\Xk_{11}.\]From this and the definition of the population dynamics, we see that 
   \begin{align}
       r_k&=r_{k-1}e^{-\gamma_{k-1}(\popgradient(\Xk)_{12}-\popgradient(\Xk)_{11})}\nonumber\\ \label{eq:rec_proof_1}
       &=r_{k-1}e^{-\gamma_{k-1}(a_\sigma \Xkon_{12}-(a_\sigma-2)\Xkon_{11})},
   \end{align}
   for $k\geq 2$. Similarly, 
   \begin{align*}
       r_1&=\frac{\Xzero_{12}}{\Xzero_{11}}e^{-\gamma_0(a_\sigma\Xzero_{12}-(a_\sigma-2)\Xzero_{11})}\\
       &=e^{-\frac{\gamma_0}{n^2}(a_\sigma-(a_\sigma-2))}\\
       &=e^{-2\frac{\gamma_0}{n^2}}.
   \end{align*}
   On the other hand, given the fact that $\Xl\in \simpn2$  we have, for any $l\in \matN$, 
   \[n\Xl_{11}+n(n-1)\Xl_{12}=1,\]
   which together with the fact $\Xkon_{12}=r_{k-1}\Xkon_{11}$, imply the following,
\begin{align*}
    \Xkon_{12}=\frac{r_{k-1}}{n(n-1)r_{k-1}+n}\\
    \Xkon_{11}=\frac{1}{n(n-1)r_{k-1}+n}.
\end{align*}
Plugging this on \eqref{eq:rec_proof_1} finishes the proof.
\end{proof}
%
\begin{proof}[Proof of Thm.\ref{thm:dyn_multistep}]
Given Lemma \ref{lem:Xk_pop_dyn}, it suffices to compare $\Xk_{11}$ and $\Xk_{12}$ for each $k>1$. Indeed, if we prove for a given $k$ that $\Xk_{12}<\Xk_{11}$ or, equivalently, that $r_k<1$, then $\Xk$ will satisfy the diagonal dominance property on all of its rows, thus we will get $\greedy(\Xk)=\id$. We now fix a $k\in [N-1]$. From Lemma \ref{lem:recursion}, it is easy to see that
\begin{equation}\label{eq:proof_recursion2}
    r_{k}=r_1\prod^{k-1}_{j=1} e^{-\gamma_j\frac{a_\sigma r_j-a_\sigma+2}{n(n-1)r_j+n}}
\end{equation}
On the other hand, notice that $a_\sigma<4$ (because $\sigma\leq 1$), from which we deduce that
\[\frac{a_\sigma r_j-a_\sigma+2}{n(n-1)r_j+n}\leq\frac{a_\sigma r_j}{n(n-1)r_j+n}<\frac4{n(n-1)},\]
for all $j\leq k-1$. This in turn implies that
\[r_k\geq e^{-2\gamma_0/n^2}\prod^{k-1}_{j=1}e^{-4\gamma_j/n(n-1)}\geq e^{-\frac4{n(n-1)}\sum^{k-1}_{j=0}\gamma_j}.\]
Given that the sequence of rates satisfies
\[\frac{n(n-1)}{4}\log{2}>\sum^{N-1}_{j=0}\gamma_j\geq \sum^{k-1}_{j=0}\gamma_j,\]
we have that 
\[r_k>\frac12\geq\frac{a_\sigma-2}{a_\sigma},\]
where the last inequality follows from the fact that $\sigma\leq 1$. 
Noticing that $k$ was arbitrary in $[N-1]$, we conclude that $r_j>\frac{a_\sigma-2}{a_\sigma}$, for all $j\in[N-1]$.
From this, and from the fact that $\gamma_j>0$, we deduce that for all $2\leq j\leq N-1$ we have that 
\[e^{-\gamma_j\frac{a_\sigma r_j-a_\sigma+2}{n(n-1)r_j+n}}<1,\]
which combined with \eqref{eq:proof_recursion2} gives that $r_k<1$, for $2\leq k\leq N$ (the fact that $r_1<1$ is evident).
\end{proof}
%
\section{Proof of Lemma \ref{lem:equivariance}}\label{sec:proofs}
    Notice that for any $\Pi\in\calP_n$ and $X\in \matR^{n\times n}$ we have
    \[\|AX-X\Pi B\Pi^T\|^2_F=\|AX\Pi-X\Pi B\|^2_F,\]
    by the unitary invariance of the Frobenius norm. Then, given the permutation invariance of the set $\calK$, if $\hat{X}\in \calS(A,B)$, then $\hat{X}\Pi^T\in \calS(A,\Pi B\Pi^T)$. It remains to show that if $\hat{X}_\calP$ is the output of the algorithm described in the lemma with input $A,B$, then the $\hat{X}_\calP\Pi^T$ will its output for the inputs $A,\Pi B\Pi^T$. For this is sufficient to notice that (given the mechanism of Algorithm \ref{alg:gmwm}) that if the largest element of $\hat{X}$ is the entry $(i,j)$ then the largest of $\hat{X}\Pi^T$ is the entry $(i,\pi(j))$, where $\pi$ is the permutation map associated to $\Pi$. 
\rev{
\section{Proof of Lemma \ref{lem:strong_convex}}\label{app:proof_strong_convex}
Assuming that $X^*=\id$, we have $B=A+\sigma Z$, where $A,Z$ are i.i.d $\operatorname{GOE}(n)$ distributed. It is a well-known fact that the $\operatorname{GOE}(n)$ measure is absolutely continuous (i.e., possesses a continuous density) w.r.t the Lebesgue measure in $\matR^{n(n+1)/2}$ (see \cite[Eq. 2.5.1]{anderson_guionnet_zeitouni_2009}). Note that, since a symmetric matrix is defined by its lower triangular part (including the main diagonal), we use the obvious identification between the space of $n \times n$ symmetric matrices and $\matR^{n(n+1)/2}$. In the following, whenever we refer to the distribution of $A$, $B$, or $Z$, we will be referring to their lower triangular part.}

\rev{Defining the event \[\mathcal{E}:=\{A,B \text{ have a common eigenvalue}\},\] we would like to establish $\prob(\mathcal{E})=0$, as this directly implies the desired conclusion, i.e., $\|AX-XB\|^2_F$ is strongly convex. Note first that, conditioning on $A$, $B$ can be regarded a an affine map of $ Z$. Then it is easy to see that if $\sigma>0$, and given that $A,Z$ are independent, the distribution of $B$ given $A$ is absolutely continuous w.r.t the Lebesgue measure in $\matR^{n(n+1)/2}$ (this follows by a simple linear change of variables). Denote $\lambda_1,\ldots,\lambda_n$ the eigenvalues of $A$, and for $i\in [n]$, define the event 
\[
\mathcal{E}'_i:=\{\lambda_i \text{ is an eigenvalue of }B \}.
\]
Clearly, $\mathcal{E}'_i=\{\operatorname{det}(B-\lambda_i \id)=0\}$. Our objective is to prove that $\prob (\mathcal{E}'_i|A)=0$. Since the distribution of $B$ given $A$ is continuous w.r.t to the Lebesgue measure, it suffices to show that $\mathcal{E}'_i$ has zero Lebesgue measure. Now, notice that $\operatorname{det}(B-\lambda_i\id)$ it is a non-constant (since $\sigma>0$) multivarite polynomial in the (lower triangular) entries of $B$. It is a known fact that the zeros of a non-constant multivariate polynomial (say in $m$ variables) has Lebesgue measure zero in $\mathbb{R}^m$. For a proof, see e.g., \cite{proof_zeros_poly}. This proves that $\prob (\mathcal{E}'_i|A)=0$. From this we deduce that $\prob (\cup_{i\in[n]}\mathcal{E}'_i|A)=0$. It is easy to see that 
\[
\prob(\mathcal{E})=\int_{\matR^{n(n+1)/2}}\prob(\cup_{i\in[n]}\mathcal{E}'_i|A)d\mu\big((A_{ij})_{i\leq j}\big)=0,
\]
where $\mu$ represents the joint distribution of the lower triangular part of $A$. From this the result follows.
}

\section{Proof of Lemma \ref{lem:proba_pos_eigenvectors}}\label{app:proof_prob_unif_sphere}
Note that for all $i\in[n]$, we have $v_i$ is marginally distributed uniformly on $\mathbb{S}^{n-1}$. This is equivalent to $v_i\stackrel{dist}{=}{\frac{g}{\|g\|_2}}$, where $g$ is a standard normal vector in $\matR^n$. It is easy to see that $\operatorname{sign}(v_i(k))\stackrel{dist}{=}\operatorname{sign}(g(k))$, and the variables $\operatorname{sign}(g(1)),\ldots,\operatorname{sign}(g(n))$ are i.i.d Rademacher distributed. Then, for any given $i\in [n]$, we have 
\begin{equation*}
    \prob(\forall k,\enskip v_i(k)>0)=\frac1{2^n}.
\end{equation*}
The result follows by applying the union bound.
%
\section{Useful calculations}\label{app:useful_results}
We gather here some useful calculations, used at different parts of the analysis.  
\begin{lemma}\label{lem:expec_GOE_sq}
Let $A$ be a $GOE(n)$ distributed random matrix, then 
\begin{align}\label{eq:expec_A2}
    \expec(A^2)&=\frac{n+1}n\id\\ \label{eq:expec_AkronA}
    \expec(A\otimes A)&= \frac1n\vec(\id)\vec(\id)^T+\frac1nT ,  
\end{align}
where $T$ is a $n^2\times n^2$ matrix that represents the transposition operator defined on $\matR^{n\times n}$.
\end{lemma}
\begin{proof}
Take $i,j\in [n]$. We have 
\begin{align*}
    \expec(A^2)_{ij}&=\sum^n_{k=1}\expec(A_{ik}A_{kj})\\
    &=\sum^n_{k=1}\frac1n\delta_{ij}(1+\delta_{ki}),
\end{align*}
from which we get \eqref{eq:expec_A2}. To prove \eqref{eq:expec_AkronA}, we take any matrix $X\in \matR^{n\times n}$ and compute
\begin{align*}
    \expec(AXA)_{ij}&=\sum^n_{k,k'=1}\expec(A_{ik}X_{kk'}A_{k'j})\\
    &= \sum^n_{k=1}\frac1nX_{kk}\delta_{ij}(1+\delta_{ki})+\frac1nX_{ji}\\
    &= \frac1n\delta_{ij}(\Tr(X)+X_{ii})+\frac1n(1-\delta_{ij})X_{ji},
\end{align*}
and from this we deduce that 
\begin{align*}
    \expec(AXA)&= \frac1n(\Tr(X)\id+X^T),\\
    &= \frac1n(\langle X,\id\rangle_F\id+X^T).
\end{align*}
Given that $(A\otimes A) \vec(X)=AXA$, it is clear that the previous implies the expression \eqref{eq:expec_AkronA}.
\end{proof}
\begin{lemma}\label{lem:population_gradmatrix}
Let $A,B\sim \calW(n,\sigma,X^*=\id)$. Recall the definition $E(X)=\|AX-XB\|^2_F$ or, in vector notation, $E(\vec(X))=\vec(X)^TH\vec(X)$, where $H=(\id\otimes A-B\otimes \id)^2$. We have 
\begin{equation}\label{eq:lemma_gradmat1}
    \nabla \expec_{A,B}E(\vec(X))=\expec_{A,B} \nabla E(\vec(X))=\expec (H)\vec(X)
\end{equation}
and
\begin{equation}\label{eq:lemma_gradmat2}
    \expec(H) = (2+\sigma^2)\big(\frac{n+1}n\big)(\id \otimes \id)-\frac2n\big(\vec(\id)\vec(\id)^T+T\big),
\end{equation}
where $T\in \matR^{n^2\times n^2}$ is the matrix representing the transpose transformation in $\matR^{n\times n}$. In particular, we have (in matrix language)
\begin{equation}\label{eq:lemma_gradmat3}
    \expec_{A,B}\nabla E(X)=(2+\sigma^2)\big(\frac{n+1}n\big)X-\frac2n\big(\Tr(X)\id+X^T\big)
\end{equation}
\end{lemma}
\begin{proof}
Both equalities in \eqref{eq:lemma_gradmat1} follow from $\expec_{A,B}(E(\vec(X)))=\vec(X)^T\expec_{A,B}(H)\vec(X)$ and linearity of expectation. On the other hand, 
a simple calculation shows that 
\begin{align*}
    \expec(H)&=\expec(\id\otimes A^2)+\expec(B^2\otimes \id)-2\expec(B\otimes A)\\
    &=\expec(\id\otimes A^2)+\expec(A^2\otimes \id)+\sigma^2\expec(Z^2\otimes \id)-2\expec(A\otimes A). 
\end{align*}
From this and Lemma \ref{lem:expec_GOE_sq}, the expression \eqref{eq:lemma_gradmat2} follows. Equality \eqref{eq:lemma_gradmat3} follows by translation of \eqref{eq:lemma_gradmat1} and \eqref{eq:lemma_gradmat2} to the matrix language.
\end{proof}
\section{Useful facts about positive semidefinite matrices}\label{app:psd_facts}
The following result proves that the entrywise exponential of a conditionally positive definite matrix is positive definite. Recall that a $n\times n$ symmetric matrix $M$ is said to be conditionally p.s.d (resp. conditionally p.d) if for all $v\in\matR^{n\times n}$ such that $v^T\ones=0$ and $v\neq 0$ we have $v^TMv\geq0$ (resp. $v^TMv>0$). 
\begin{theorem}\label{thm:condp_to_exppd}
Let $M\in \matR^{n\times n}$ be  conditionally p.s.d. Then the following statements hold.
\begin{enumerate}[label=(\roman*)]
    \item $\hadexp(tM)$ is p.s.d. for all $t\geq 0$.
    \item If, in addition, $M$ is conditionally p.d, then $\hadexp(tM)$ is p.d. for all $t>0 $.
    \item $\hadexp(M)$ is p.d. if and only if $M_{ii}+M_{jj}>2M_{ij}$ for all $i\neq j$.
\end{enumerate}
\end{theorem}
 Theorem \ref{thm:condp_to_exppd} parts $(i)$ and $(iii)$ are part of \cite[Lemma 2.5]{REAMS1999}. Notice that part $(ii)$ follows from part $(iii)$. 
\begin{lemma}\label{lem:hadprod_pd_psd}
    Let $A,B\in \matR^{n\times n}$ be symmetric matrices. If $A$ is p.s.d with nonzero diagonal entries, and $B$ is p.d., then $A\odot B$ is p.d.
\end{lemma}
\begin{proof}
    Denote $\lambda_{\operatorname{min}}(B)$ to be the smallest eigenvalue of $B$ (which is strictly positive by assumption). Then by the Schur product theorem we have that 
    \begin{align*}
        A\odot(B-\lambda_{\operatorname{min}}(B)\id)\succcurlyeq 0
      \iff  A\odot B \succcurlyeq\lambda_{\operatorname{min}}(B)\diag(A_{11},A_{22}.\ldots,A_{nn})\id.
    \end{align*}
    Since $\lambda_{\operatorname{min}}(B)>0$ and the entries $A_{11},\ldots,A_{nn}$ are nonzero, we conclude that $A\odot B\succ 0$.
\end{proof}
\begin{lemma}\label{lem:integral_psd}
    Let $M(t)\in \matR^{n\times n}$ be a symmetric matrix which is also integrable w.r.t the Lebesgue measure over the positive reals. 
    If $M(t)$ is p.s.d (resp. p.d) for $t>0$, then the matrix $\int^\infty_0M(t)dt$ is p.s.d (resp. p.d).
\end{lemma}
\begin{proof}
    For any $v\in\matR^n$ we have from the integrability of $M(t)$ that
    \begin{equation*}
        v^T\left(\int^\infty_0 M(t)dt\right)v=\int^\infty_0 v^TM(t)vdt
    \end{equation*}
    from which the statement follows directly.
\end{proof}
%
%
\end{document}